\documentclass[11pt]{article}
\usepackage{amssymb, authblk}
\usepackage{amsmath,bbm}
\usepackage{fullpage}
\usepackage{amsthm, hyperref}
\usepackage{graphicx}
\usepackage{verbatim}
\usepackage{mathtools}
\usepackage[dvipsnames]{xcolor}
\usepackage{tikz}
\usetikzlibrary{arrows.meta}
\usetikzlibrary{decorations.pathreplacing}
\DeclareMathAlphabet{\mathpzc}{OT1}{pzc}{m}{it}
\usepackage{natbib}
\usepackage{pdflscape, capt-of}
\usepackage{comment}
\usepackage[algo2e]{algorithm2e}
\usepackage{algorithm}
\usepackage{algpseudocode}

\usepackage{hyperref}[]
\hypersetup{
    colorlinks=true,
    linkcolor=blue,
    filecolor=magenta,      
    urlcolor=cyan,
      citecolor=blue,
    }

\usepackage{cleveref}

\newtheorem{theorem}{Theorem}

\newtheorem{proposition}[theorem]{Proposition}
\newtheorem{corollary}[theorem]{Corollary}

\newtheorem{remark}{Remark}
\newtheorem{assumption}{Assumption}

\allowdisplaybreaks

\DeclareMathOperator*{\argmax}{arg\,max}

\newcommand{\CPMAB}{\textbf{SelectiveBandits}}
\newcommand{\CPDMAB}{\textbf{PrudentBandits}}
\newcommand{\gap}{\mbox{(gap)}}

\title{Generalized non-stationary bandits}
\author[1]{Anne Gael Manegueu}
\author[1]{Alexandra Carpentier}
\author[2]{Yi Yu}
\affil[1]{Institute for mathematical stochastics, Otto-von-Guericke University of Magdeburg}
\affil[2]{Department of Statistics, University of Warwick}
\date{}

\begin{document}

\maketitle
\begin{abstract}%

In this paper, we study a non-stationary stochastic bandit problem, which generalizes the switching bandit problem. On top of the switching bandit problem (\textbf{Case a}), we are interested in three concrete examples: (\textbf{b}) the means of the arms are local polynomials, (\textbf{c}) the means of the arms are locally smooth, and (\textbf{d}) the gaps of the arms have a bounded number of inflexion points and where the highest arm mean cannot vary too much in a short range. These three settings are very different, but have in common the following: (i) the number of similarly-sized level sets of the logarithm of the gaps can be controlled, and (ii) the highest mean has a limited number of abrupt changes, and otherwise has limited variations. We propose a single algorithm in this general setting, that in particular solves in an efficient and unified way the four problems (a)-(d) mentioned.

\end{abstract}

\section{Introduction}
\label{sec:intro}
The multi-armed bandit (MAB) problem is an online learning problem with partial feedback, where an agent sequentially makes decisions in the presence of uncertainty. At each time $t$ until the end of the budget $T$, the agent chooses an \textit{arm} $k_t$ among $K$ arms, and receives a reward $X_{k_t,t}$ corresponding to this arm. The most common objective for the agent is to maximize her cumulative reward, namely the sum of rewards that she collected.

In the standard stochastic MAB setting, each arm is characterized by a reward distribution, and every time that the arm is sampled, the agent receives an independent sample from this distribution, see~\citet{bubeck2012regret, lattimore2018bandit} for surveys. However, in many real-world applications, the best arm might change over time, and more generally the environment might be non-stationary. A large part of the bandit literature - on adversarial bandits, see~\cite{bubeck2012regret, lattimore2018bandit} also for surveys - is devoted to the question of finding strategies that track the oracle stationary policy\footnote{I.e.~that aims to have a cumulative reward not much smaller than the best (oracle) stationary policy.}, namely the policy that would always pull the arm that has the best cumulative reward. This literature makes few assumptions on the data generating process - typically only assuming that the rewards are bounded. However, this literature usually tracks a small class of policies, against which the agent competes.  Another part of the literature is devoted to the question of comparing the agent's policy to more complex policies, i.e.~comparing it to the oracle non-stationary policy in a class of non-stationary policies, under more restrictive assumptions. This is the setting that we will study in this paper.

As for the research on non-stationary MAB where one tracks oracle non-stationary strategies, we can divide it into two main categories that correspond to two main settings: the \textit{rested bandit setting} \citep[see e.g.][]{10.2307/2335176, bouneffouf2016multi, bouneffouf2014contextual, levine2017rotting, seznec2020single} and the \textit{restless bandit setting} \citep[see e.g.][]{gafni2018learning, liu2012learning, meshram2018whittle, besson2018doubling, cheung2019learning, russac2019weighted, wei2017tracking, seznec2020single}. In the rested bandit setting, the distribution of an arm may change when it is sampled, while in the restless bandit setting, the distribution of an arm may vary as a function of the current time. In this paper, \textit{we focus on the restless bandit setting}. 
In particular, the means of the arm distributions change as a function of time. We write $\mu_k(t)$ for the mean of arm $k$ at time $t$. In this setting, we evaluate the performance of an algorithm through the expected cumulative regret where we compare to the \textit{best non-stationary policy}: $R_T = \sum_t \max_k \mu_{k}(t) - \sum_t \mathbb E (X_{k_t,t})$, i.e.~we compare the performance of the agent to that of an oracle who would know \textit{at each time step} which arm has the best mean. 

The most commonly studied restless bandit setting is the piecewise stationary setting, or switching bandit setting. We refer to this setting in this paper as \textbf{Case a)}. In this setting, it is assumed that the number of times that the arm means are allowed to change is bounded. More precisely, the time steps $[T] = \{1, \ldots, T\}$ at which the agent makes a decision, are partitioned into $M$ intervals delimited by \textit{change points} $\bar \tau_m$ such that $1 = \bar \tau_1 < \ldots < \bar\tau_{M} < \bar\tau_{M+1} = T+1$. In the piecewise stationary setting, or switching bandit setting, the means $\mu_k(t)$ are assumed to be constant on $[\bar\tau_m, \bar\tau_{m+1})$ - yet might vary across intervals corresponding to another index $m' \neq m$. 
To the best of our knowledge, this setting was first formally studied in \citet{garivier2011upper}, where two methods the sliding-window upper confidence bounds (SW-UCB) and the discounted UCB \citep[see also][]{kocsis2006discounted} are proposed, providing an optimal (up to logarithmic terms) regret of order $\sqrt{MKT}$.  The paper~\cite{auer2002finite} also introduced an algorithm that would work well in this setting, although the setting was not introduced in this paper. These seminal works led to a rich line of work termed the \textit{switching bandits}, studied by \cite{allesiardo2015exp3}, \citet{hartland2007change}, \citet{yu2009piecewise}, \citet{liu2018change} and  \citet{cao2019nearly}, \citet{besson2018doubling} among others. Recently, the paper~\cite{auer2019achieving} provided an algorithm that also adapts to the number of change points $M$. And \cite{seznec2020single} studied among other things adaptivity to $M$ in the case where the arm means are monotone.

Beyond piecewise stationary settings, other type of nonstationary settings have also been studied. A seminal paper for non-stationary bandits is~\citet{besbes2014stochastic}, that proposes a very flexible setting for dealing with non-stationary environments. In this paper, it is assumed that the sum over time of the supremum of the amount of deviation on the mean of the arms is bounded by $V_T$. They prove that the minimax regret in this case is - up to logarithmic terms - $(KV_T)^{1/3} T^{2/3}$. \citet{cheung2019learning, russac2019weighted} provided an extension of this setting in the linear bandits framework, and~\cite{chen2019new} provided a method hat is adaptive to $V_T$. In \cite{burtini2015improving}, a modification to the linear model of Thompson sampler is proposed and deals with a general non-stationary means but lacks theoretical supports.  \cite{combes2014unimodal} considered nonstationary means, which are Lipschitz continuous and unimodal.

In this paper, we study a more general framework of changes.  To be specific, in addition to a few abrupt changes in the means of the arms, there could also be \textit{on top of that} some less relevant changes in the means \textit{at all times} - for instance in the important cases where the means of the arms are well-approximated locally by a polynomial, or when the means of the arms are locally smooth, or when the \textit{arm gaps} have a bounded number of inflexion points. Motivated by these examples, we propose in this paper an extension of the switching bandit setting that takes into account in an unified way such cases - which we describe more in details below. But before doing that, we summarize more precisely the main examples (\textbf{Cases b)} to \textbf{d)}) that motivated our work below. 

\noindent \textbf{Case b)} In this case we assume that the means of all arms are local polynomials. Namely, we can define functions $f_k$'s over $[0,1]$ such that (i) for any $t \in [T]$, and any $k$, we have $\mu_k(t) = f_k(t/T)$ and (ii) for any $k$, we can partition $[0,1]$ in $M^*$ intervals such that $f_k$ is a local polynomial of degree at most $\gamma^*$ over these intervals, and the sum in absolute value of the coefficients of such a polynomial is upper bounded by $u^*$. We prove in this case an upper bound on the expected regret of order
    $$K\sqrt{TM^*(\gamma^* + 1)}\log^{3/2}(T) + Ku^*.$$
    This case encompasses, in particular, the case where the arm means are locally linear ($\gamma^* = 1$) or locally quadratic ($\gamma^* = 2$). In both cases, we obtain here a regret of same order up to logarithmic terms as in the switching bandit setting when \textit{the number of change-points in each mean is bounded by $\bar M$} - which implies at most $M=\bar MK$ intervals where all means are constant in the switching bandit setting. And so, an extension of switching bandits to e.g.~locally linear case or locally quadratic means of the switching bandit setting is not harder than the classical switching setting.
    
\noindent \textbf{Case c)} In this case we assume that the means of all arms is locally smooth. Namely, we can define functions $f_k$'s over $[0,1]$ such that (i) for any $t \in [T]$, and any $k$, we have $\mu_k(t) = f_k(t/T)$ and (ii) for any $k$, we can partition $[0,1]$ in $M^*$ intervals such that the function $f_k$ is $\alpha$-H\"older smooth over each interval. We prove in this case an upper bound on the expected regret of order
    $$ \sqrt{KTM^*}\log (T) + (K\log T)^{\frac{2\alpha}{2\alpha+1}} T^{\frac{\alpha+1}{2\alpha+1}}.$$
    In the case where the functions $f_k$'s are $\alpha$-H\"older smooth on the entire interval $[0,1]$, we get a bound of order $(K\log (T))^{\frac{\alpha}{2\alpha+1}} T^{\frac{\alpha+1}{2\alpha+1}}$, which is, when it comes to the dependence in $T$, of order $T$ times the minimax estimation rate of $\alpha$-H\"older smooth functions.

\noindent \textbf{Case d)} We first introduce the \textit{gaps} $\Delta_k(t) = \max_{k'} \mu_{k'}(t) - \mu_k(t)$ . In this case we assume that the gaps $\Delta_k(t)$ have a number of \textit{inflexion} points bounded by $\upsilon^*-1 \geq 0$. On top of this, the optimal mean $\max_k \mu_k(t)$ is not allowed to vary too quickly - i.e.~the function $\max_k \mu_k(t)$ is such that its deviations are bounded by some $B^* \geq 0$ on any time interval of length bounded by $K$. We prove in this case a bound on the expected regret of order
    $$K\sqrt{T\upsilon^*} \log^{3/2}(T) +  TB^*.$$
    Whenever $B^*$ is smaller than $K/\sqrt{T}\log^{3/2}(T)$, - not too much short-term variation in the optimal mean - we obtain here a rate of same order up to logarithmic terms as in the switching bandit setting when \textit{the number of change-points in each mean} is bounded by $\upsilon^*-1$. Our assumptions are much weaker than the one for switching bandits: we just assume that the gaps have a bounded number of inflexion points, and that the only smoothness assumption is imposed on the optimal mean.

All these cases are extensions of the switching bandit setting, but are covered by the assumptions in~\cite{besbes2014stochastic}. However, the only result which can be partially recovered by the algorithm in~\cite{besbes2014stochastic} is \textbf{Case c)}. Regarding the other cases, it is clear that the other cannot be recovered as the rate in~\cite{besbes2014stochastic} is of larger order than $T^{2/3}$ in all of our settings.

Throughout this paper, our assumptions on the gaps $\Delta_k(\cdot)$ and means $\mu_k(\cdot)$ can be summarized as follows: (i) for some $B^* \geq 0$, the optimal mean $\max_k \mu_k(.)$ cannot vary by more than $B^*$ on any interval of length $K$ intersected with $[\bar\tau_m, \bar\tau_{m+1})$, and (ii) over the intervals $[\bar\tau_m, \bar\tau_{m+1})$, the gaps $\Delta_k(.)$ are allowed to vary only by an amplitude that is proportional to themselves\footnote{I.e.~for any $t,t' \in  [\bar\tau_m, \bar\tau_{m+1})$, we have either $\frac{1}{2^{1/4}} \leq \frac{\Delta_k(t)}{\Delta_k(t')} \leq 2^{1/4}$, or $\Delta_k(t) \leq 2(T^{-1/2} \lor B^*)$.}. These assumptions are stated precisely in Assumption~\ref{As:2-unobs1} and are satisfied in all four \textbf{Cases a)} to \textbf{d)} for some specific choice of parameters. We provide an algorithm, \CPDMAB, whose regret is of order
$$\sqrt{KTM} \log T + TB^*.$$
This algorithm is related to existing algorithms for switching bandits - in that it contains a method for detecting change points, and that it forces regular exploration of arms that have been ruled out as sub-optimal, to avoid the event that they become optimal without the algorithm realising it. However a main difference with respect to existing strategies for switching bandits is on the change point detection procedure: first, we do not aim at detecting change points in the means, but just in the \textit{gaps}, and second, we do not aim at detecting all the change points in the gaps - as there might be too many - but just the ones that are \textit{significant enough} for the sub-optimal arms to become optimal.\\
In the switching bandit setting described in \textbf{Cases a)}, the regret of our algorithm is as needed of order $\sqrt{KTM} \log T$. For the three other cases \textbf{Cases b)} to \textbf{d)}, the regret is as described above - when adjusting the parameters of the algorithms to the cases, as described in Section~\ref{sec:conse}.

 The outline of this paper is as follows. In Section~\ref{Gapnotobserved}, we describe the setting of this paper. In Section~\ref{algorithm}, we describe our algorithm \CPDMAB, and we provide bounds on its cumulative regret in Section~\ref{regretanalysis}. In Section~\ref{sec:conse}, we provide some corollaries of the aforementioned regret bounds in all \textbf{Cases a)} to \textbf{d)}. Finally in Section~\ref{sec:related Work}, we discuss our results and provide some research directions. For space reason, all the proofs are in the appendix, as well as an alternative setting of interest, in which we also provide results, and in which, instead of observing noisy evaluations of means of the arms, we observe noisy evaluations of the means of the gaps - see Appendix~\ref{sec-gap-obs-model} for the setting and Appendix~\ref{regretanalysisgap2} for the associated results.

\section{Setting}
\label{Gapnotobserved}

We consider the \textit{stochastic non-stationary bandit setting}, where the rewards observed by the learner for its actions, are noisy version of the true mean. In this setting, the learner disposes of a finite set of arms $[K]=\{1, \ldots, K\}$ to sequentially play during $T$ time steps. at each time $t$ in $[T]=\{1, \ldots, T\}$.  At each time $t$, each arm $k \in [K]$ is characterised by an \textit{unknown} reward distribution with mean $\mu_k(t)=f_{k}(t/T)$, where $f_k(\cdot)$ is an \textit{unknown function} defined on $[0, 1]$. Moreover the (not necessarily observed) reward 
at time $t \in [T]$ and for arm $k\in [K]$ is defined as:
    \begin{equation}\label{eq-obs-model-01}
        X_{k,t}= \mu_{k}(t) +\epsilon_t^{(k)},
    \end{equation}
 where $\epsilon_t^{(k)}$ is a random noise. Let us write $k^*_{t}$ for a best arm at time~$t$, i.e.~$k^*_{t} \in \argmax_{k\in [K]} \mu_k(t)$. We define the gap between an arm $k$ and a best arm at a given time $t$ as $\Delta_{k}(t) = \mu^{*}(t) - \mu_{k}(t)$, where $\mu^{*}(t) =  \mu_{k^*}(t)$ is the mean of arm $k^*_{t}$.
 
 At each time $t\in [T]$, the learner chooses an arm $k_t \in [K]$ and \textit{observes the reward} which corresponds to arm $k_t$, namely $X_{k_t,t}= \mu_{k_t}(t) +\epsilon_t^{(k_t)}$. 
  Throughout this section, we assume the following. 
  \begin{assumption}[Independent and bounded reward values]\label{As:1-unobs1}\label{As:1-unobs}
The rewards $(X_{k,t})_{k\in [K], t\in [T]}$ defined in Equation~\eqref{eq-obs-model-01} belong to $[0,1]$, and are independent both across arms and across time. In addition, for any $k\in [T]$ and $t\in [T]$, $\mathbb E(X_{k,t}) = \mu_{k}(t)$, i.e.~the noise $\epsilon_t^{(k)}$ is centered.
\end{assumption}

This assumption is common in the existing literature \citep[e.g.][]{liu2017change, cao2019nearly, garivier2011upper}. Instead of assuming that the $X_{k,t}$ belong to $[0,1]$, we could have assumed that they are sub-Gaussian, without changing the setting or the algorithms that will be provided later.

\noindent \textbf{Goal: regret minimization.} The goal of this paper is to maximize the expected cumulative reward after $T$ time steps, i.e.~$\sum_{t = 1}^T \mathbb E(X_{t}) = \sum_{t = 1}^T \mathbb{E} (X_{k_t, t})$.

 We compute the regret compared to an \textit{optimal non-stationary policy} that chooses an optimal arm at each time $t$, i.e.~$k^*_t \in \argmax_{k \in \mathcal{K}} \mu_k(t)$.  The regret can therefore be formulated as
 \[
        R_T = \sum_{t=1}^T  \max_{k \in [K]} \mu_k(t) - \mathbb E \sum_{t=1}^T X_t = \sum_{t=1}^T  \left\{\max_{k \in [K]} \mu_k(t) - \mu_{k_t}(t) \right\} = \mathbb E \sum_{t=1}^T \Delta_{k_t}(t),
    \]
 
\noindent \textbf{The non-stationary environment.} 
In this paper, we consider a non-stationary environment, where the gaps $\Delta_k(t)$ of each arm $k\in [K]$, and optimal mean $\mu^*(t)$, may change (abruptly) at unknown time instants called \textit{change-points} or \textit{breakpoints}, and may also change less abruptly between any two change-points. We first introduce the \textit{change-points} which are integers (unknown to the agent) such that $1 = \bar \tau_1<  \ldots < \bar \tau_M < \bar \tau_{M+1} = T+1$, where $M \in \mathbb N^*$. Note that $[T]= \bigcup_{m= 1}^{M} [\bar\tau_m, \bar\tau_{m+1}) \cap \mathbb N$, and we refer to the intervals $([\bar\tau_m, \bar\tau_{m+1}))_{m \in \{0,\dots, M\}}$ as the 
    \textit{connex components}. 
    
The following assumption concerns the optimal mean of the arms $\mu^*(\cdot)$, and the amount by which it is allowed to vary.
\begin{assumption}\label{As:2-unobs1}\label{As:2-unobs}
There exists $B^* \geq  0$ such that the following holds. For any $t, t' \in [\bar \tau_m, \bar \tau_{m+1})$, with $|t - t'| \leq K$, we have $|\mu^*(t) - \mu^*(t')| \leq B^*$.
\end{assumption}

Through Assumption~\ref{As:2-unobs1}, we first assume that between any two subsequent change-points, the value of the optimal mean $\mu^*(.)$ can change at most up to $B^*$ within $K$ time steps. This implies that $\mu^*$ may change abruptly at the change-points $\bar \tau_m$, and that on any interval $[\bar \tau_m, \bar \tau_{m+1})$, $\mu^*$ might change as well, but in a controlled way. Namely, inside any $[\bar \tau_m, \bar \tau_{m+1})$ and within a window of time of size $K$, only changes of magnitude less than $B^*$ are tolerated for  $\mu^*$ - on the other hand, the index of best arm $k^*_t$ can change at any time in an arbitrary fashion.

This assumption is clearly satisfied in the context of switching bandits (\textbf{Case a)}) - as the value of the best mean does not change over $[\bar \tau_m, \bar \tau_{m+1})$ - and also for \textbf{Cases b)}, \textbf{c)} and \textbf{d)}.

We now provide the following assumption which restricts the variations of the gaps between two change-points.
\begin{assumption}
\label{As:2-sec-obs-gap1} \label{As:2-sec-obs-gap}
There exists $B^* \geq 0$ such that the following holds. For any $k \in [K]$ and any $m \in \{1, \ldots, M\}$, for any $t,t' \in [\bar\tau_m, \bar\tau_{m+1}) \cap \mathbb{N}$:
    \begin{align*}
        2^{-1/4} \leq \Delta_k(t)/\Delta_k(t') \leq 2^{1/4}  \quad \mbox{or} \quad \Delta_k(t) \in [0, \, 2 (T^{-1/2} \lor B^*)].
\end{align*}
\end{assumption}
It can be interpreted as follows: on the intervals $[\bar\tau_m, \bar\tau_{m+1})$, (i) either the gap of arm $k$ does not deviate by more than a multiplicative constant time itself, or (ii) the gap is and remains small. In other words, this is an assumption on the number of connex components.  The level sets sizes form a geometric sequence, therefore the number of level sets is a logarithmic function of the arm gaps.

\section{Algorithm}
\label{algorithm}

In this section, we propose a phase-based MAB algorithm, termed the \CPDMAB~detailed in Algorithm~\ref{alg:select}. The algorithm functions by actualising and sampling an \textit{active arm set}. In parallel, the algorithm looks for \textit{significant change point}. When such a significant change point is detected, the algorithm resets its knowledge and starts a new \textit{episode}. This algorithm includes a change point detection subroutine and enforces a limited exploration of arms that have been discarded from the active set. This is related to existing algorithms for switching bandits, but there are some important differences - in particular in the significant change point detection routine - that we will highlight.

The \CPDMAB~algorithm is executed in \textit{rounds}, which are grouped into the aforementioned episodes - where we remind that the episodes are delimited by the detected significant change points.

More precisely, in each episode: the algorithm divides time into \emph{rounds} indexed by $r$, the duration of which corresponds to the cardinality of the active arms set at round $r$, that we write $K_{r} \subset \mathcal{K}$. The active arm set $K_r \subset [K]$ consists not only of arms that have not been ruled out as sub-optimal at the current round, but also of sub-optimal arms fulfilling the selection conditions. We will elaborate this later.
All arms in $K_{r}$ are sampled once at round $r$. 
For simplicity, we assume that in every round, the arms in the active arm set are sampled consecutively in the increasing order of arm indices.

At the end of round $r$, the arms are also subject to the change point detection mechanism $\mathrm{CP}_r$ - see Algorithm~\ref{alg:changepoint}, we will also describe this more in details later on. If a change point is detected, then a new episode starts and the algorithm resets its information: the starting point of round $r$ is then declared to be an estimated change point, written $\rho_{M_r} = r$. This marks the end of episode $M_r$, where $M_r$ is the index of the current episode. Thereafter, a \emph{new episode} $M_{r+1}= M_r + 1$ starts. The algorithm resets its knowledge in that it does not take into account any data that it collected before the detected significant change point.

The algorithm terminates at the end of the budget at round $R$, which is  defined to be $R = \min \left\{r: \, \sum_{r' = 1}^r |K_{r'}| \geq T \right\}$.  We remark that $R$ is a random variable.  Due to the limitation of the total budget, it is possible that not every arm in $K_R$ will be pulled.  Since the goal of this paper is to minimize the regret, for simplicity, we assume that $\sum_{r = 1}^R |K_r| = T$. 

\begin{remark}
Note that, despite the reindexation of rounds and episodes, the formulation stated in Section~\ref{sec-gap-obs-model} can easily be recovered by the fact that each round $r$ starts at time
    \[
        t_r = \begin{cases}
            \sum_{r' = 1}^{r-1} |K_{r'}|+1, & r \geq 2, \\
            1,& r = 1.
        \end{cases}
    \]
    For $k \in K_r$, we write $\bar{t}_r(k) \in [t_r, t_{r+1})$ as the time instant at which arm $k$ is sampled in round $r$, and $\tilde X_r(k) = X_{\bar t_r(k)}$ as the corresponding observation. By the reindexation of rounds and episodes, we have
    \[
        \mathbb E R_T = \mathbb E \sum_{r \leq R} \sum_{k\in K_r} \max_{k' \in [K]}\mu_{k'}(\bar t_r(k')) -\mathbb E \sum_{r \leq R} \sum_{k\in K_r} \tilde X_r(k) =  \mathbb E \sum_{r \leq R} \sum_{k\in K_r} \Delta_k(\bar t_r(k)).
    \]
\end{remark}

\paragraph{Estimation of the gaps.} 
Our algorithm \CPDMAB~is mainly based on estimators of the gaps. This is performed as follows. For any integer pair $(r, r')$, $1\leq r < r' \leq R$, we define the set of arms being pulled in every round in $\{r',...,r-1\}$ as $ \mathcal S(r', r)= \left\{ k \in [K]: \, \prod^{r-1}_{j=r'} \mathbf 1\{k \in K_j\}=1\right\}$.  The motivation behind the definition of $\mathcal S(r', r)$ is as follows. In each episode, the algorithm will keep in the active set all arms that could be optimal based on the collected information, and will sample them at each round - i.e.~they will stay in the active set active set $K_{r''}$ for any $r'' \in [r,r')$. If $[r,r')$ is included in a single episode, $\mathcal S(r', r)$ therefore contains these arms.
    
    For any arm $k$ pulled at time $\bar{t}_r(k)$ in round $r$, the gap-comparison of arm $k$ relative to arm $k'\in \mathcal S(r',r)$ is assessed by the quantity
    \begin{equation}\label{eq-delta-kk'-def}
        \hat \Delta_{k}^{(k')} (r',r) = \begin{cases}
            \frac{1}{T_k(r',r)}\sum_{j=r'}^{r-1} \left\{\tilde X_j(k') - \tilde X_j(k)\right\}\mathbf 1\{k\in K_j\}, & T_k(r',r)>0, \\
            0, & T_k(r',r)=0,
        \end{cases}
    \end{equation}
    where $T_k(r', r)$ is the number of pulls of arm k in rounds $\{r',..., r\}$ and $\tilde{X}_j(k) = X_{\bar{t}_j(k)}$, with $\bar{t}_j(k) \in [t_j, t_{j+1})$. Then we can define the estimated gap of arm $k$ based on the informations computed on the time interval $[r', r)$ by
    \begin{equation}\label{eq-delta-k-unobs}
        \hat \Delta_{k}(r',r) = \max_{k' \in \mathcal S(r', r)} \hat \Delta_{k}^{(k')}(r',r).
    \end{equation}

\paragraph{Estimation of the significant change points.}  A main difference between our work and existing literature on the switching bandit problem is in the change point detection procedure. Here we assume the weaker Assumption~\ref{As:2-sec-obs-gap}, which bounds the number of 'significant' changes in the gaps. In this paper, (i) we therefore look for change points in the estimated gaps rather than in the estimated means, and (ii) we only aim at detecting changes that are 'large enough'. We make this now more precise. 
Between any two consecutive change points $\bar{\tau}_m$ and $\bar{\tau}_{m+1}$, for any $k \in [K]$, it holds from Assumption~\ref{As:2-sec-obs-gap} that

    \[
        \max_{t \in [\bar{\tau}_m, \bar{\tau}_{m+1})} \Delta_k(t) - \min_{t \in [\bar{\tau}_m, \bar{\tau}_{m+1})} \Delta_k(t) \leq \max\left\{2 \min_{t \in [\bar{\tau}_m, \bar{\tau}_{m+1})} \Delta_k(t), \, 2 \left(T^{-1/2} \lor B^*\right)\right\}.
    \]
    Based on this and on our estimator of the gaps from Equation~\eqref{eq-delta-k-unobs}, we declare the existence of a change point if there exists $k \in [K]$ satisfying \eqref{Chp:eq1}.

\begin{algorithm}[htpb]
\caption{Change point detection $\mathrm{CP}_r((\tilde X_r(k))_{k\in K_r})$}
\label{alg:changepoint}
\SetAlgoLined
\textbf{Input:} a sequence $(\tilde X_j(k))_{j\in \{\rho_{M_{r-1}, \ldots, r-1}\}, k\in K_r}$  \\
 \For{$k \in K_r$}{
  for any $u, v, u', v' \in [\rho_{M_{r - 1}}, r)$, $u < v$, $u' < v'$\\
  \eIf{\vskip -1cm\begin{align}
    |\hat \Delta_k(u,v) - \hat \Delta_k (u', v')|\geq 2 \hat \Delta_k(u,v) + 2  \sqrt{\frac{2\log(2KT^3)}{T_k(u,v) \land T_{k}(u',v')}} + 2 B^* \label{Chp:eq1}
     \end{align}\vskip -1cm}{
   $\mathrm{CP}_r = 1$
   }{
   $\mathrm{CP}_r = 0$
  }
 }
\end{algorithm}

\paragraph{Construction of the active arms sets $K_r$.} To complete the description of \CPDMAB, we now provide details on the construction of the active arms sets.

The active arm set $K_r$ is composed of two types of arms. First, it contains all arms that have not been ruled out - based on their gap estimators from Equation~\eqref{eq-delta-k-unobs} - as sub-optimal. Second, and since we are in a non-stationary environment, we need to ensure that arms that have been ruled out as sub-optimal are nevertheless regularly explored, in order to detect changes in their means which could lead to them becoming optimal. This is related to what is commonly the case in the switching bandit literature.

In order to define the set of active arms, we first define the following quantities. For any $r \in \{1, \ldots, R\}$ and $k \in [K]$, let $N_k(r)$ be the time distance between the starting point $t_r$ of the round $r$ and the time point when the arm $k$ was pulled for the last time, i.e.
   is it not rather that \[
        N_k(r) = \begin{cases}
            t_r - \max \{t < t_r: \, \mathbf 1\{ k_t= k\}= 1\}, & \{t < t_r: \, \mathbf 1\{ k_t= k\} = 1\} \neq \emptyset, \\
            t_r, & \{t < t_r: \, \mathbf 1\{ k_t= k\} = 1\} = \emptyset.
        \end{cases}
    \]
    
    Note that this quantity consist of time instants, and is actualised for all arms at the beginning of each round $r$. The exploration is done by comparing $N_k(r)$ with an estimator $\tilde{N}_k(r)$, defined as
    \begin{equation}\label{eq-n-tilde-definition}
        \tilde N_k(r) = \begin{cases}
            0, & \tilde \Delta_k(\rho_{M_{r-1}},r')=0, \,  \forall r' \in (\rho_{M_{r-1}}, r]  \\
            \hat{\Delta}_k(\rho_{M_{r-1}},r') \sqrt{\frac{TK}{M}},  & r' = \min \{ r''\in (\rho_{M_{r-1}}, r]:\, \tilde \Delta_k(\rho_{M_{r-1}},r'')>0 \},
        \end{cases}
    \end{equation}
    where, for $B^*>0$, we define
    \begin{equation}\label{eq-tilde-Delta-definiton}
        \tilde \Delta_k(r',r) = \begin{cases}
            \left\{\hat \Delta_k(r',r) - \sqrt{\frac{2\log(2KT^3)}{T_k(r', r)}} - 2B^*\right\} \lor 0, & T_k(r',r) > 0, \\
            0, & T_k(r',r) = 0.
        \end{cases}
    \end{equation}
 which is a lower bound on the gap between arm $k$ and any arm $k'\in \mathcal{S}(r', r)$. Note that we have $\tilde N_k(r)=0$ if and only if $\tilde \Delta_k(\rho_{M_{r-1}},r)=0$ . 
  
  Now getting back to the selection rule for the set of active arms $K_r$, this set is composed of all arms such that $\tilde N_k(r) \leq N_k(r)$. In other words,  an arm $k \in [K]$ is sampled at round $r$, if at least $\tilde N_k(r)$ steps passed since its last pull in the current episode.
  Thus, for any round $r$ - where we recall that the corresponding episode starts at the estimated change point $\rho_{M_{r-1}}$ - the active set $K_r$ is constructed in the following way. First of all, all arms without sub-optimality evidence - such that $\tilde \Delta_k(\rho_{M_{r-1}},r)=0$ - are automatically selected in $K_r$ since $\tilde N_k(r)=0$. Second, \CPDMAB~samples also some of the other arms - that have been ruled out as sub-optimal, namely such that $\tilde \Delta_k(\rho_{M_{r-1}},r)>0$ - when the condition that $\tilde N_k(r) \leq N_k(r)$ is fulfilled, namely when they have not been sampled for a long time. Note that, the larger the lower bound $\tilde \Delta_k(\rho_{M_{r-1}}, r)$, the larger the distance $\tilde N_k(r)$ - this is related to what is done in the switching bandit setting.   
 
\begin{algorithm}[htbp]
\caption{ \CPDMAB  }
\label{alg:select} 
\SetAlgoLined
{\bf Parameters:} $M \in \mathbb{N}^*$, $B^* > 0$\\
{\bf Initialsation:} $M_0 = 1, \rho_{0} = 1$, $r\leftarrow 1$, $(\tilde N_k(r))_k = 0$, $(N_k(r))_k = 0$,  $(s_k)_k = 0$\\
\While{$\sum_{u \leq r} |K_{u}|  \leq T$}{
    \For{$k \in \mathcal{K}$}{
        $N_k(r) = t_r - s_k$ 
    }
    Set $K_r = \{k \in \mathcal{K}: \, \tilde N_k(r) \leq N_k(r)\}$ \\
    \For{$k \in K_r$}{
    Sample arms in $K_r$ once each in the increasing order of their indices\\
    $s_k = t_r(k)$}
    \If{$\mathrm{CP}_r = 1$}{
        $M_r = M_{r-1}+1$, $\rho_{M_r} =r$, $(\tilde N_k(r))_k = 0$
    }
    \uElse{
    $M_{r} = M_{r-1}$
    }
    $r \leftarrow r+1$                                                    
}
\end{algorithm}

\section{Regret Analysis}
\label{regretanalysis}

In this section, we provide upper bounds on the expected regret of the algorithm \CPDMAB~in the setting of Section~\ref{Gapnotobserved}, where we assume that the observations are noisy evaluations of the arm means. The strategy is presented in Algorithm~\ref{alg:select}, with a sub-routine detecting change points detailed in Algorithm~\ref{alg:changepoint} and its theoretical guarantees provided below in Theorem~\ref{th:2}.

\begin{theorem}\label{th:2}
Assume that Assumptions~\ref{As:1-unobs}, \ref{As:2-unobs} and \ref{As:1} hold for $M,B^*>0$. The expected regret of \CPDMAB~launched with parameters $M$ and $B^*$  is such that, with a large enough absolute constant $C > 0$,
    \begin{align}\label{reg2}
        \mathbb E R_T \leq C\log(T) \sqrt{KTM} + CTB^*. 
    \end{align}
\end{theorem}
This theorem is proved in Appendix~\ref{sec:1}. Theorem~\ref{th:2} states that the expected cumulative reward is upper bounded as where the second term in \eqref{reg2} comes from the bound on the variation of the optimal gap in Assumption~\ref{As:2-sec-obs-gap}. This exhibits an interesting phase transition.  If $B^* \leq T^{-1/2}$, then the penultimate term in \eqref{reg2} is of order $T^{1/2}$, which is dominated by $\sqrt{KTM}$.  If $B^* > T^{-1/2}$, then the penultimate  term in \eqref{reg2} is of order $TB^*$.  Furthermore, if $\log(KT) \sqrt{KM/T} \leq B^*$, then $TB^* \gtrsim \log(KT) \sqrt{KTM}$.  As a summary, we have that
    \[
         \mathbb E R_T \lesssim \begin{cases}
            \log(T) \sqrt{KTM}, & B^* \leq \log(T) \sqrt{KM/T}, \\
            TB^*, & B^* > \log(T) \sqrt{KM/T}.
         \end{cases}
    \]
    This general result allows for the number of arms $K$, and the number of significant change points $M$  
    to diverge, as the total budget $T$ grows unbounded. $B^*$ is given to the algorithm as a parameter and must be such that Assumption~\ref{As:2-sec-obs-gap} holds. Note however that, for any $B^* \geq 0$, there exists a $M$ such that this assumption holds - and that the larger $B^*$, the smaller the corresponding $M$. So that in the end, $B^*$ is more to be interpreted as a choice by the learner - in terms of the level of tolerance for small variations in near-optimal arms that she is ready to tolerate - rather than as a parameter.
    
    However the algorithm takes, on top of $B^*$, (an upper bound on) $M$ as a parameter. This is much more constraining as the number of significant change point is typically unknown. We discuss this more in details in Section~\ref{sec:related Work}.

\paragraph{Now regarding optimality.} \textit{The switching bandits}, which is an MAB problem with the arm means are assumed to be constant between two change-points, has already been aforementioned as a special case of our setting, and in particular, referred to as the special \textbf{Case a)} of our setting. This case will be discussed in detail in section~\ref{sec:conse}, where it will be showed that the \CPDMAB~run with parameter $B^*$ achieves a regret of order $\mathcal{O}(\sqrt{TKM})$, which is up to logarithmic terms of same order as the minimax regret bound for switching bandits~\cite{auer2002nonstochastic,
garivier2011upper} - which is a sub-case of our setting.  
 
\section{Consequences in several models}
\label{sec:conse}
Recalling the motivating examples in Section~\ref{sec:intro}, in this section we see how our results from Section~\ref{regretanalysis} can lead to specific results in these interesting MAB settings.

\subsection{Case a: Switching bandits.}\label{ss:casea}

Arguably, the simplest and most popular non-stationary MAB problem is the switching bandit problem. In this setting the expected rewards are allowed to change abruptly at some given times - the change points - but have to remain constant between two consecutive change points. Thus the switching bandit setting is a specific case of the setting of Section~\ref{Gapnotobserved}, where on top of Assumptions~\ref{As:2-unobs} and \ref{As:2-sec-obs-gap}, we have that $B^* = 0$ and that the $\mu_k$ are piecewise stationary- and therefore $\Delta_k$'s - are constant on $([\bar\tau_m, \bar\tau_{m+1}))_{m \in \{1,\dots, M\}}$. 
This problem can therefore be immediately solved by our approach, and Theorem~\ref{th:2} implies immediately the following corollary.

\begin{corollary}[Switching bandits]\label{cor-1}
Assume that Assumption~\ref{As:1-unobs} holds. Assume that there exist $1= \bar \tau_1< \bar \tau_2 < \ldots < \bar \tau_M<\bar \tau_{M+1} = T+1$ such that on any interval $[\bar\tau_m, \bar\tau_{m+1})$ indexed by $m$, the arm means $\mu_k(.)$ are constant. The expected cumulative regret of \CPDMAB~run with parameters $M > 0$ and $B^*=0$ satisfies that, with $c > 0$ being an absolute constant,
    \[
        \mathbb E R(T)\leq c \log(T) \sqrt{TKM}.
    \]
\end{corollary}
We recover classical result for the switching bandit setting, see e.g.~\cite{garivier2011upper}.

\subsection{Case b: Switching polynomial expected rewards}

Consider a non-stationary MAB setting where $f_k$, that corresponds to the mean $\mu_k$, is a local polynomial function of degree at most $\gamma^*\geq 0$, on at most $M^* \geq 1$ intervals. To be specific, we assume that for any $k$, there exist $1= \bar \tau_1^{(k)}< \bar \tau_2^{(k)} < \ldots < \bar \tau_{M^*}^{(k)}<\bar \tau_{M^*+1}^{(k)} = T+1$ such that on any interval $[\bar\tau_m^{(k)}/T, \bar\tau_{m+1}^{(k)}/T)$ indexed by $m,k$, the functions $f_k(\cdot)$ associated to the arm means $\mu_k(\cdot)$ are polynomials of degree at most $\gamma^*$. We write $u_{m,k,0}, \ldots, u_{m,k,\gamma^*}$ for the associated coefficients - so that for any $x \in [\bar\tau_m^{(k)}/T, \bar\tau_{m+1}^{(k)}/T)$, we have $f_k(x)  = \sum_{0\leq r \leq \gamma^*} u_{m,k,r} x^r$. We also assume that there exists $u^* \geq 0$ such that, for any $m,k$, we have $\sum_{0\leq r \leq \gamma^*} |u_{m,k,r}| \leq u^*$.

We can apply Theorem~\ref{th:2} to provide a bound in this case on the regret of \CPDMAB - see Appendix~\ref{proof:conse} for the proof.
\begin{corollary}[Switching polynomial means] \label{cor-4}
Assume that Assumption~\ref{As:1-unobs} holds. Assume there exist $M^*\geq 1, \gamma^* \geq 0$ and $u^* \geq 0$ such that for any $k \in [K]$, there exists $M^*$ intervals such that the function $f_k(\cdot)$ is a polynomial of degree at most $\gamma^*$ on each of these intervals, and the $\ell_1$-norm of its coefficient vector is upper bounded by $u^*$ on each interval - see the beginning of the subsection for more details.

The expected regret of \CPDMAB~run with parameters $M= M^*(\gamma^*+1) K(\lfloor \log_2(T^{1/2})\rfloor +1)$ and $B^* = u^{*}(K/T)$~satisfies, with $c > 0$ being an absolute constant,
    \[
        \mathbb E R(T) \leq cK\sqrt{TM^*(\gamma^*+1)}\log^{3/2}(T) + c Ku^*.
    \]
\end{corollary}
\vskip -0.3cm
\begin{remark}
In the case of local polynomial with small degree with coefficients bounded by $1$ - e.g.~$\gamma = 1$ for piecewise linear functions and $u^* \leq 2$ - the regret is bounded by $K\sqrt{TM^*}$, i.e.~is of same order as the regret of switching bandits \textit{in the case where each mean $\mu_k$ is constant over $M^*$ intervals} - giving rise to at most $M = M^*K$ intervals over which \textit{all means $\mu_k$ are constant at the same time}.
\end{remark}

\subsection{Case c: Switching smooth expected rewards.} 

Consider a non-stationary MAB setting with a function mean reward $f_k$ being a locally $\alpha$-H\"older smooth function on $M^*$ intervals for some $\alpha \leq 1, M^*\geq 1$.  To be specific, we assume that for any $k\in [K]$ there exist $1= \bar \tau_1^{(k)}< \bar \tau_2^{(k)} < \ldots < \bar \tau_{M^*}^{(k)}<\bar \tau_{M^*+1}^{(k)} = T+1$ such that on any interval $[\bar\tau_m^{(k)}/T, \bar\tau_{m+1}^{(k)}/T)$ indexed by $m$, the function $f_k(.)$ associated to the arm means $\mu_k(.)$ is $\alpha$-H\"older smooth, that is to say, that for any $u,v \in [\bar\tau_m/T, \bar\tau_{m+1}/T)$, we have $|f_k(u) - f_k(v)| \leq |u-v|^{\alpha}$.

We can apply Theorem~\ref{th:2} to provide a bound in this case on the regret of \CPDMAB - see Appendix~\ref{proof:conse} for the proof.
\begin{corollary}[Switching smooth expected rewards]\label{cor:c}
Assume that there exist $M^*\geq 1$ and $\alpha \leq 1$ such that for any $k \in \mathcal{K}$, there exist $M^*$ intervals that form a partition of $[0,1]$ and such that the function $f_k(\cdot)$ restricted to any of these intervals is an $\alpha$-H\"older smooth function - see the beginning of the subsection for more details.  

The expected regret of \CPDMAB~run with parameters $M = M^* + K(B^*)^{1/\alpha}$ and $B^* = (K \log(T)/T)^{\frac{2\alpha}{2\alpha+1}}$ satisfies, with $c > 0$ being an absolute constant,
    \begin{align*}
        \mathbb E R(T) & \leq c \log T\sqrt{KTM^*} + c (K \log(T))^{\frac{2\alpha}{2\alpha+1}} T^{\frac{\alpha+1}{2\alpha+1}}.
    \end{align*}
\end{corollary}

\begin{remark} Note that if the functions $f_k$ associated to the rewards are smooth over the entire interval $[0,1]$, then $M=1$ and the regret is bounded in order by $(K\log T)^{\frac{2\alpha}{2\alpha+1}} T^{\frac{\alpha+1}{2\alpha+1}}$.  The dependence in $T$ is up to log terms of order $T^{\frac{\alpha+1}{2\alpha+1}}$.  Recalling that the minimax estimation error of an $\alpha$-H\"{o}lder continuous function $f$, in terms of the $L^2$-norm, is of order $T^{\alpha/(2\alpha+1)}$.  Then the term $T^{\frac{\alpha+1}{2\alpha+1}}$ can be seen as $T$ times the optimal non-parametric estimation error rate.\\
Note also that in the case where $M=1$, this result can be recovered by the results in~\cite{besbes2014stochastic} - since here the H\"older assumption is mainly used to control the amount of deviation that each gap can experience, which is very similar to the quantity $V_T$ introduced in~\cite{besbes2014stochastic}.
\end{remark}

\subsection{Case d: Finite number of inflexion points in the expected gaps, and locally stable optimal mean.} \label{ss:cased}

Consider a non-stationary MAB setting with means $\mu_k$ and associated function $f_k$ defined on $[0, 1]$ satisfying the following, for some $B^* \geq 0$ and $\upsilon^*\geq 1$: (1) for any $t,t' \in [T]: |t-t'| \leq K$, we have that $|\mu^*(t) - \mu^*(t')| \leq B^*$, and 
(2) for any $k \in [K]$, the sub-optimality gaps $\Delta_k(\cdot)$ has a number of inflexion points bounded by $\upsilon^*-1$, namely for any $k$, there exists change points $1=\bar \tau_1^{(k)} < \bar \tau_2^{(k)} < \ldots < \bar \tau_{\upsilon^*}^{(k)}< \bar \tau_{\upsilon^*+1}^{(k)} = T+1 $ such that $\Delta_k(\cdot)$ is monotone on $[\bar \tau^{(k)}_m, \bar \tau^{(k)}_{m+1})$.

We can apply Theorem~\ref{th:2} to provide a bound in this case on the regret of \CPDMAB - see Appendix~\ref{proof:conse} for the proof.
\begin{corollary}[Finite number of inflexion points in the expected gaps, and locally stable optimal mean]\label{cor21}
Assume that for $\upsilon^*\geq 1$, the gaps $\Delta_k(\cdot)$ are have at most $\upsilon^*-1$ inflexion points. We also assume for any $t,t'\in [T]$ such that $|t-t'| \leq K$, we have that $|\mu^*(t) - \mu^*(t')| \leq B^*$ - see the beginning of the subsection for more details. 

The expected regret of \CPDMAB~run with parameters $M = \upsilon^*K(\lfloor \log_2(T^{1/2})\rfloor +1)$ and $B^* > 0$ satisfies, with $c > 0$ being an absolute constant,
    \[
        \mathbb E R(T) \leq cK\sqrt{T\upsilon^*} \log^{3/2}(T) + cB^*T, 
    \]
\end{corollary}

\begin{remark} Note that, as long as $\mu^*(\cdot)$ is not too rough, e.g.~$B^* \leq K/\sqrt{T}$, the upper bound is of order $K\sqrt{\upsilon^*T}$. This is of the same order as the corresponding bound in switching bandits where we assume that all $\mu_k$ are constant by parts over at most $\upsilon^*$ intervals - leading to a partition in approximately $K\upsilon^* \log T$ intervals, when we merge the change points of all arms. So we do not lose anything up to logarithmic terms to the switching bandits case in this much more complex case.
\end{remark}

\section{Discussion}
\label{sec:related Work}

\paragraph{A new non-stationary setting and a general algorithm.} We introduce a novel non-stationary framework, that generalises in a flexible way the switching bandit framework. We propose an algorithm for this setting, \CPDMAB, that flexibly adapts to several interesting \textbf{Cases a)--d)} - all solved by the same algorithm, with an adaptation of its parameters. \textbf{Cases b)--d)} are novel to the best of our knowledge, and based on the existing literature, only \textbf{Case a)} and \textbf{Case c)} can be recovered - and for \textbf{Case c)}, only in a subcase. The main idea of our new setting is to focus on the gap's level sets of exponential size, and to bound the number of such level sets - resulting in an algorithm \textit{that detects only significant change points in the gaps}. While some aspects in our algorithm and its analysis is related to the switching bandit setting, there are quite a few significant additional challenges, emerging from the need of a new change point detection procedure.

\paragraph{Setting where we observe noisy evaluations of the gaps.} In this paper, we also propose in Appendix~\ref{sec:Gapobserved} an alternative setting where the learner collects noisy evaluations of the gaps. For this, we provide an algorithm, \CPMAB, that achieves similar results to \CPDMAB, but is much simpler - see Appendix~\ref{regretanalysisgap2}. This algorithm takes just $B^*$ as parameter, namely, it adapts to $M$.

\paragraph{Adaptivity to $M$.} Algorithm \CPDMAB~requires (upper bounds on) $M$ and $B^*$. While $B^*$ can be seen as an algorithmic choice, which calibrates the tolerance level of the algorithm to small changes, requiring knowledge on $M$ is a strong assumption. Relaxing this assumption is an open question in our setting, but very interesting recent results~\citep{auer2019achieving, chen2019new, seznec2020single} in the switching bandit literature that solved this problem opens interesting perspectives for our setting. However, adapting the algorithm in~~\cite{auer2019achieving} to our setting represents a significant challenge, as we allow the optimal mean to change \textit{at all time}. This complicates significantly the estimation of the gaps in our setting - which will be even more problematic when it comes to the strategy of~\cite{auer2019achieving}, which has to sample many consecutive times sub-optimal arms.

As mentioned in the last paragraph, we also propose in Appendix~\ref{sec:Gapobserved} an alternative setting where the learner collects noisy evaluations of the gaps. In this setting, \CPMAB, achieves similar results to \CPDMAB but that takes just $B^*$ as parameter. However, the setting where we observe noisy evaluations of the gaps is way less challenging, since effectively in this setting we can always decipher whether an arm is optimal or not by just sampling it - and we do not need to compare it to others.

\section*{Acknowledgements}
The work of A.~Manegueu is supported by the DFG CRC 1294 `Data Assimilation', Project A03. The work of A.~Carpentier is partially supported by the Deutsche Forschungsgemeinschaft (DFG) Emmy Noether grant MuSyAD (CA 1488/1-1), by the DFG - 314838170, GRK 2297 MathCoRe, by the FG DFG, by the DFG CRC 1294 `Data Assimilation', Project A03, and by the UFA-DFH through the French-German Doktorandenkolleg CDFA 01-18 and by the SFI Sachsen-Anhalt for the project RE-BCI.  The work of Y.~Yu is supported by DMS-EPSRC: EP/V013432/1.

\bibliographystyle{ims}
\bibliography{all_refs}

\appendix

\section{Learning with noisy observations of the gap}
\label{sec:Gapobserved}

We first consider a specific bandit setting where when sampling an arm, we observe a noisy version \textit{of the gaps}. This is equivalent to assuming that the mean of the current optimal arm is $0$, and that while the index of the optimal arm might change with time, its value remains the same and is equal to $0$. This setting differs from the classical bandit setting, and makes the problem significantly simpler. But it can be nevertheless interesting in practice - either when we now, by the nature of the problem, the mean of the current best arm, or e.g.~in the context where instead of aiming at always sampling the best arm, we aim at sampling an arm whose mean is larger than a given threshold, and measure regret according to this threshold.

\subsection{Learning model}\label{sec-gap-obs-model}

We consider a non-stationary stochastic MAB with a set of arms $[K]$ and a sequence of time steps $[T]$, where the expected reward of each arm may change over time. In particular, we assume that at time $t \in \mathcal{T}$, arm $k\in \mathcal{K}$ is characterized by its mean reward $\mu_k(t)=f_{k}(t/T)$, where $f(\cdot)$ is an unknown function. Define at each time $t\in \mathcal{T}$, and for any arm $k\in \mathcal K$ the rewards\footnote{They will be observed at time $t$ for any $k$, more precisely only one for a specific $k$ will be observed at time $t$, see Equation~\eqref{eq-obs-model-1}.}
    \begin{equation}\label{eq-obs-model-0}
        X_{k,t}= - \Delta_{k}(t) +\epsilon_t^{(k)},
    \end{equation}
    where $\epsilon_t^{(k)}$ is the noise variable and $\Delta_{k}(t)$ denotes the gap between arm $k$ and the best arm at time~$t$, arm~$k^*_{t}$ - which might also depend of time. To be specific, $\Delta_{k}(t) = \mu^{*}(t) - \mu_{k}(t)$, where $\mu^{*}(t) = \argmax_{k\in \mathcal{K}} \mu_k(t)$, is the mean reward of arm $k^*_{t}$. And at each time $t\in \mathcal{T}$, an agent chooses an arm $k_t \in \mathcal{K}$ and observes the reward
    \begin{equation}\label{eq-obs-model-1}
        X_{k_t,t}= - \Delta_{k_t}(t) +\epsilon_t^{(k_t)}.
    \end{equation}

Throughout this section, we assume the following.

\begin{assumption}
\label{As:1}
The rewards $(X_{k,t})_{k\in \mathcal{K}, t\in\mathcal{T}}$ defined in \eqref{eq-obs-model-0} are independent both across arms and across time steps. In addition, for any $k\in \mathcal{K}$ and $t\in \mathcal{T}$, it holds that $-X_{k,t} \in [0,1]$ and $\mathbb E X_{k,t} = -\Delta_k(t)$ - the noise is centered.
\end{assumption}
Following the convention of existing literature \citep[e.g.][]{liu2017change, cao2019nearly, garivier2008upper}, Assumption~\ref{As:1} first assumes that the rewards process $(X_{k,t})_{k\in \mathcal{K}, t\in \mathcal{T}}$ consist of independent random variables, which suggests that, the sequence of random noise variables $(\epsilon_t^{(k)})_{k,t}$ is also drawn independently.  Furthermore, Assumption~\ref{As:1} specifies that, for some action $k\in \mathcal{K}$ taken at time $t\in \mathcal{T}$, $\mathbb E(X_{k,t})=-\Delta_{k}(t)$; in other words, the expected observable payoff is an unbiased noisy observation of the true gap $\Delta_{k}(t)$.  

\medskip
\noindent \textbf{Goal: regret minimization.} The goal of this paper is to maximize the expected cumulative reward after $T$ time steps, i.e.
    \[
        \sum_{t = 1}^T \mathbb E(X_{t}) = \sum_{t = 1}^T \mathbb{E} (X_{k_t, t}).
    \]
    Since the rewards observed are noisy versions of negative gaps, at each time step, we observe a negative version of the instantaneous regret. Thus, maximizing the cumulative reward is equivalent to minimizing the cumulative regret, which is defined as
    \begin{align}
        R^{\gap}_T = - \mathbb E \sum_{t=1}^T X_t = \mathbb E \sum_{t=1}^T \Delta_{k_t}(t).\label{R:1}
    \end{align}

\medskip
\noindent \textbf{Environment: change points.} In this paper, we consider the case of a non-stationary environments, where the gaps $\Delta_k(t)$ of each arm $k\in \mathcal{K}$ may change (abruptly) at unknown time instants called change points.  In addition, between two consecutive change points, we allow $\Delta_k(t)$ to slowly vary - as stated in Assumption~\ref{As:2-sec-obs-gap}, from Section~\ref{Gapnotobserved}. But here, we do not assume anything on the best mean $\mu^*$ - i.e.~we will not assume Assumption~\ref{As:2-unobs}.

\subsection{Algorithm}

To tackle the MAB problem specified in Assumptions~\ref{As:1} and \ref{As:2-sec-obs-gap}, we propose the \CPMAB, which is detailed in  Algorithm~\ref{alg:select1}.  The central idea and the intuition behind this algorithm consist of two components: (i) exploring the arms with no sub-optimality evidence as much as possible in order to maximize the reward, and (ii) detecting changes in the gaps, to prevent wasting resources on the arms turning sub-optimal.

Algorithm~\ref{alg:select1} works as follows. The algorithm functions by actualising and sampling an \textit{active arm set}, i.e.~the set of arms that have not been ruled out as sub-optimal at the current time. This continues as long as this active set is not empty. When on the other hand this set is empty, this means that a \textit{significant change point} occurred - the algorithm subsequently resets its knowledge and starts a new \textit{epsisode}. 
More precisely, in each episode: the algorithm divides time into \emph{rounds} indexed by $r$, the duration of which corresponds to the cardinality of the active arms set at round $r$, that we write $K_{r} \subset \mathcal{K}$. 
Provided that $K_{r} \neq \emptyset$, all arms in $K_{r}$ are sampled once at round $r$.  
For simplicity, we assume that in every round, the arms in the active arm set are sampled consecutively in the increasing order of arm indices. 
If $K_{r}=\emptyset$, i.e.~if all the arms have been ruled out as sub-optimal, then there is evidence that a significant change-point occurred. The starting point of round $r$ is then declared to be an estimated change point written $\rho_{M_r} = r$ - which marks the end of episode $M_r$, where $M_r$ is the index of the current episode. Thereafter, a new \emph{episode} $M_{r+1}= M_r + 1$ starts, the algorithm resets, and all arms are included back in the active arm set. 

The algorithm terminates at the end of the budget at round $R$, which is  defined to be
    \[
        R = \min \left\{r: \, \sum_{r' = 1}^r |K_{r'}| \geq T \right\}.
    \]
    We remark that $R$ is a random variable.  Due to the limitation of the total budget, it is possible that not every arm in $K_R$ will be pulled.  Since the goal of this paper is to minimize the regret, for simplicity, we assume that $\sum_{r = 1}^R |K_r| = T$. 

\begin{remark}
Note that, despite the reindexation of rounds and episodes, the formulation stated in Section~\ref{sec-gap-obs-model} can easily be recovered by the fact that each round $r$ starts at time
    \[
        t_r = \begin{cases}
            \sum_{r' = 1}^{r-1} |K_{r'}|+1, & r \geq 2, \\
            1,& r = 1.
        \end{cases}
    \]
    For $k \in K_r$, we write $\bar{t}_r(k) \in [t_r, t_{r+1})$ as the time instant at which arm $k$ is sampled in round $r$, and $\tilde X_r(k) = X_{\bar t_r(k)}$ as the corresponding observation. By the reindexation of rounds and episodes, we have
    \[
        R^{\gap}_T = -\mathbb E \sum_{r \leq R} \sum_{k\in K_r} \tilde X_r(k) =  \mathbb E \sum_{r \leq R} \sum_{k\in K_r} \Delta_k(\bar t_r(k)).
    \]
\end{remark}

We now introduce the criterion for checking arms' optimality.  For any $k \in [K]$ and any integer pair $(r', r)$, $1 \leq r' < r \leq R$, we denote the corresponding gap estimator that takes into account the data collected for arm $k$ between round $r'$ and round $r$
    \begin{equation}\label{eq-delta-k-obs}
        \hat \Delta_{k}(r',r) = \begin{cases}
            -\frac{1}{T_k(r',r)}\sum_{j=r'}^{r-1} \tilde X_j(k)\mathbf 1\{k\in K_j\}, & T_k(r', r) > 0, \\
            0, & T_k(r', r) = 0,
        \end{cases}
    \end{equation}
    where 
    \begin{equation}\label{eq-Tk-def}
        T_k(r', r) = \sum_{j=r'}^{r-1} \mathbf 1\{k \in K_j\}
    \end{equation}
    is the number of pulls of arm $k$ in rounds $\{r', \ldots, r-1\}$.  In addition, for $B^{*} \geq 0$, let 
    \[
        \tilde \Delta_k(r',r) = \begin{cases}
            \left\{\hat \Delta_k(r',r) - \sqrt{\frac{\log(2KT^3)}{2T_k(r',r)}} - 2 \left(\frac{1}{\sqrt{T}} \lor B^*\right)\right\} \lor 0, & T_k(r', r) > 0, \\
            0, & T_k(r', r) = 0,
        \end{cases}
    \]
    as the corresponding lower bound on the estimated gap - where we correct both for high probability lower deviations, and also for sensibility to deviations that are too small.
    
    Algorithm~\ref{alg:select1} resets the algorithm when a significant change point is detected, i.e.~a new episode starts.  For any round $r$, the current episode is denoted as $M_r$, which starts at the most recently estimated significant change point $\rho_{M_{r-1}}$.  The set of active arms at round $r$ in this episode is then computed based on the information obtained between $\rho_{M_{r-1}}$ and $r$ - namely based on the lower bounds on the gaps $\tilde{\Delta}_k(\rho_{M_{r-1}}, r)$ computed with the corresponding information.  To be specific, $K_r$ is the set of arms such that $\tilde{\Delta}_k(\rho_{M_{r-1}}, r) = 0$, i.e.~that could be optimal based on the collected samples. 
\begin{algorithm}[htpb]
\caption{$\CPMAB$}
\label{alg:select1} 
\SetAlgoLined
{\bf Parameters:} $B^* > 0$, $T \in \mathbb{N}^*$ \\
{\bf Initialization:} $M_0 = 1$, $\rho_{0} = 1$, $r=1$, $K_1 = \mathcal{K}$\\
\While{$\sum_{j \leq r} |K_{j}| \leq T$}{
\If{$r > 1$}{
Set $K_r = \{k \in \mathcal{K}: \, \forall \rho_{M_{r-1}}\leq  r'\leq r, \, \tilde \Delta_k(\rho_{M_{r-1}},r') = 0\}$
}
\If{$K_r \neq \emptyset$}{
$M_{r} = M_{r-1}$\\
Sample arms in $K_r$ consecutively in increasing order of arm indices
}
\uElse{
$M_{r} = M_{r-1}+1$\\
$\rho_{M_r} =r$}
$r = r+1$
}
\end{algorithm}

\section{Regret Analysis in the setting where we observe noisy evaluations of the gap - see Sections~\ref{sec:Gapobserved}}
\label{regretanalysisgap2}

In this section, we provide upper bounds on the expected regret of the algorithm \CPMAB , in the setting described in Section~\ref{sec:Gapobserved}.

We consider the setting of Section~\ref{sec:Gapobserved}, where we assume that the observations are noisy evaluations of the negative gaps. The strategy is detailed in Algorithm~\ref{alg:select1}, with its theoretical guarantees provided below in Theorem~\ref{th:1}.

\begin{theorem}\label{th:1}
Assume that Assumptions~\ref{As:2-sec-obs-gap} and \ref{As:1} are satisfied for some $M,B^*>0$. The expected regret of \CPMAB~launched with the parameter $B^*$ satisfies 
    \begin{align}
        \mathbb E R^{\gap}_T \leq \{2^{3/2} + 2^{1/2} c\log(2K T^3)\} \sqrt{KTM} + 2KM + 8T(B^{*} \lor T^{-1/2})+1, \label{eq:4}
    \end{align}
    for any $c \geq 16$.
\end{theorem}
This bound is of the same order as the one in Theorem~\ref{th:2}, which holds in the more challenging setting where our observations are noisy evaluation of the means (and not of the gaps). However a main advantage of algorithm \CPMAB\,over Algorithm \CPDMAB~- on top of its relative simplicity - is the fact that it does not take (an upper bound on) $M$ as a parameter. This is a very desirable property, as discussed in Section~\ref{regretanalysis}.

Note that from Theorem~\ref{th:2}, all corollaries of Section~\ref{sec:conse} hold also in the setting - replacing \CPDMAB~with \CPMAB~and replacing $\mathbb E R_T$ with $\mathbb E R^{\gap}_T$, in addition to the difference that $\CPMAB$ will just need $B^* \geq 0$ as the sole parameter. It is quite interesting to have an algorithm that does not take an (upper bound on) $M$ as a parameter, and it is what we provide in this more restrictive setting.

\section{Proof of Theorem~\ref{th:1}} \label{Proof:th1}

The proof is divided into several steps.  In Section~\ref{sec:IntnoChange}, we propose a partition of all rounds $\{1, \ldots, R\}$ to ease notation for the rest of the proof.  In Section~\ref{sec-large-prob}, we define a large probability event, in which the rest of the proof will be conducted.  In Section~\ref{sec-no-spurious-jump}, it will be shown that Algorithm~\ref{alg:select1} will not over-estimate the number of change points.  In Section~\ref{sec-low-pulls}, we control the numbers of pulling sub-optimal arms.  Finally, we reach the conclusion in Section~\ref{Proof:th1:regret}.

\subsection{Dividing rounds into intervals without changes} \label{sec:IntnoChange}

Note first that Assumption~\ref{As:2-sec-obs-gap} implies that for any $k \in [K]$ and any $m \in \{1, \ldots, M\}$, there exists $\kappa_{k,m}\in \mathbb N \cup \{-\infty\}$ such that for any $t \in [\bar\tau_m, \bar\tau_{m+1}) \cap \mathbb{N}$:
    \begin{align}\label{eq:assump}
        \Delta_k(t) \in \left[2^{\kappa_{k,m}/2}\left(T^{-1/2} \lor B^*\right), \, 2^{\kappa_{k,m}/2+1}\left(T^{-1/2} \lor B^*\right) \right] \, \mbox{or} \, \Delta_k(t) \in \left[0, \, 2\left(T^{-1/2} \lor B^*\right)\right].
\end{align}

Without loss of generality, we assume that the connex components $([\bar \tau_m, \bar \tau_{m+1}))_{m\leq M}$ are ordered and non-overlapping, with a strictly increasing sequence $(\bar \tau_{m})_{m\leq M}$.  If $([\bar \tau_m, \bar \tau_{m+1}))_{m\leq M}$ have overlapping, then one can always refine the partition to obtain a non-overlapping sequence of intervals satisfying Assumption~\ref{As:2-sec-obs-gap}.  For any connex component $[\bar \tau_m, \bar \tau_{m+1})$, define
    \[
        \tau_m = \min\{r: t_r \geq \bar\tau_m\},
    \]
    as the round starting right after the true change point $\bar \tau_m$.  Note that while the sequence $(\bar \tau_m)_{m\leq M}$ consists of time instants, $(\tau_m)_{m\leq  M}$ is a sequence of round indices.  Now given the sequence of estimated change points $(\rho_{m})_{r\in \{1,..., M_R\}}$, define $\tilde M$ as
    \begin{equation}\label{eq-m-tilde-def}
        \tilde M = \left| \{(\rho_{m})_{m\leq M_R}, ( \tau_m)_{m \leq M+1}, ( \tau_m -1)_{m \leq M+1} \} \right|\leq 2M + M_R + 2,
    \end{equation}
    i.e.~$\tilde{M}$ is the cardinality of the set consisting of the starting rounds of detected change points $\rho_m$'s, the rounds right after true change points $\tau_m$'s, and the rounds right before the true change points $(\tau_m-1)$'s.

Based on the set defined in \eqref{eq-m-tilde-def}, we write $(I_u)_{u\leq \tilde M}$ as its elements, reindexed in an increasing order.  For any $1\leq u\leq \tilde M$, let
    \[
        I_u = [i_{u}, i_{u+1}),
    \]
    where $i_{\tilde{M} + 1} = R+1$.  Therefore, $(I_u)_{u\leq \tilde M}$ form a partition of $\{1, \ldots, R\}$ and is a refinement of the connex components $([\bar \tau_m, \bar \tau_{m+1}))_{m\leq M}$  by the detected change points $(\rho_{m})_{m\leq M_R}$ and $(\tau_m-1)$'s.  
    
Let    
    \begin{equation}\label{eq-m1-defi}
        \mathcal{M}_1=\{u\leq \tilde M: \,  I_u \neq [\tau_m -1, \tau_m), \quad \forall 1 \leq m \leq M+1\}
    \end{equation}
    and 
    \begin{equation}\label{eq-m2-defi}
        \mathcal{M}_2=\{u\leq \tilde M :  I_u = [\tau_m -1, \tau_m), \quad \exists m \leq M+1 \}.
    \end{equation}
    For any $ u \in \mathcal{M}_1$, there exists $m'\in \{1, \ldots, M\}$ such that $I_u \subset [\bar\tau_{m'-1}, \bar\tau_{m'})$. Thus, for any $u \in \mathcal{M}_1$ and $k\in [K]$, following Assumption~\ref{As:2-sec-obs-gap}, we define
    \[
        \bar \Delta_{k}(u) = 2^{\kappa_{k,m'}/2} (T^{-1/2} \lor B^*)
    \]
    as the lower bound of $\Delta_k(t)$ in all rounds in $I_u$.

\subsection{The large probability event}\label{sec-large-prob}

For any $\delta \in [0, 1]$, let the clean event $\xi(\delta)$ be
    \begin{align*}
        & \xi(\delta) = \Bigg\{\Big|\hat \Delta_k(r',r) - \frac{1}{T_k(r',r)} \sum_{j=r'}^{r-1} \Delta_{k}(\bar t_j(k)) \mathbf 1\{k \in K_j\} \Big| \leq \sqrt{\frac{\log(2/\delta)}{2T_k(r',r)}}, \\
        & \hspace{6cm} \forall k \in [K], \, \forall 1 \leq r' < r \leq R, \, T_k(r', r) > 0  \Bigg\}.
    \end{align*}
        
\begin{proposition}\label{prop:clean_event}
Under Assumptions~\ref{As:2-sec-obs-gap} and \ref{As:1}, for any $\delta \in [0, 1]$, the clean event $\xi(\delta)$ holds with probability at least $1 - KT^2\delta$.
\end{proposition}

\begin{proof}
By definition, for any $k \in [K]$ and any integer pair $(r', r)$, $1 \leq r' < r \leq R$, we have that
    \begin{align*}
        & \left|\hat \Delta_k(r',r) - \frac{1}{T_k(r',r)} \sum_{j=r'}^{r-1} \Delta_{k}(\bar t_j(k)) \mathbf 1\{k \in K_j\} \right| \\
        = & \left|\frac{1}{T_k(r',r)} \sum_{j=r'}^{r-1} \{-\tilde X_j(k)+\Delta_{k}(\bar t_j(k)) \}\mathbf 1\{k \in K_j\} \right|\\
        = & \left|\frac{1}{T_k(r',r)} \sum_{j=r'}^{r-1} -\epsilon^{(k)}_{\bar t_j(k)} \mathbf 1\{k \in K_j\} \right|
    \end{align*}
    and that
    \[
        \left\{ \sum_{j=r'}^{s} -\epsilon^{(k)}_{\bar t_j(k)} \mathbf 1\{k \in K_j\}\right\}_{r' \leq s \leq r-1}
    \]
    is a martingale adapted to the filtration $\sigma\{\epsilon^{(k')}_{\bar{t}_j(k')}\}_{u \leq r-1, \, k' \in [K]}$.  Moreover, due to Assumption~\ref{As:1}, the increments $(- \epsilon^{(k)}_{\bar t_j(k)})_{r'\leq j\leq r}$ are bounded in $[-1, 1]$. Therefore, for any $\delta \in [0, 1]$, it follows from the Azuma--Hoeffding inequality \citep[e.g.~Corollary~2.20 in][]{wainwright2019high} that with probability at least $1-\delta$,
    \[
        \left|\frac{1}{T_k(r',r)} \sum_{j=r'}^{r-1} -\epsilon^{(k)}_{\bar t_j(k)} \mathbf 1\{k \in K_j\} \right|\leq \sqrt{\frac{\log(2/\delta)}{2T_k(r',r)}}.
    \]
    The proof is completed following applying a union bound argument.
\end{proof}

\subsection{No spurious jump detected}\label{sec-no-spurious-jump}

\begin{proposition}\label{prop:jump}
Under Assumptions~\ref{As:2-sec-obs-gap} and \ref{As:1}, for any $m \in  \{1,..., M_R\}$ and any $\delta \in [0, 1]$, we have that on the event $\xi(\delta)$, $[\rho_{m}, \rho_{m+1}]$  is not included in any connex component $[\tau_{m'},  \tau_{m'+1})$, $m' \leq M$. 
\end{proposition}

\begin{proof} 
For notational simplicity, we let $\xi = \xi(\delta)$.  We prove by contradiction and assume that for a fixed $m \in  \{1,..., M_R\}$ there exists $m'\leq M$ such that $[\rho_m, \rho_{m+1}]\subset [ \tau_{m'}, \tau_{m'+1})$.  By the definition of $\mathcal{M}_1$, this implies that there exists $ u\in \mathcal{M}_1$, such that $[\rho_m, \rho_{m+1}) = I_u$.  Since there is no true change point in this interval, let arm $k^*$ be such that $\bar \Delta_k^*(u)=0$ and
    \[
        \Delta_{k^*}(\bar t_r(k^*)) \in [0, 2 (T^{-1/2} \lor B^*)], \quad \forall r \in I_u.
    \]
    Then on the event $\xi$, for any $r \in I_u$ we have that
    \begin{align*}
        \hat \Delta_{k^*}(\rho_{m}, r) &\leq \frac{1}{T_{k^*}(\rho_{m}, r)}\sum_{j=\rho_m}^{r-1} \Delta_{k^*}(\bar t_j(k^*))) \mathbf 1\{k^* \in K_j\} + \sqrt{\frac{\log(2/\delta)}{2T_{k^*}(\rho_{m}, r)}} \\
        &\leq 2 (T^{-1/2} \lor B^*)+\sqrt{\frac{2\log(2/\delta)}{2T_{k^*}(\rho_{m}, r)}},
    \end{align*}
    which implies that
    \[
        \hat \Delta_{k^*}(\rho_{m}, r)-\sqrt{\frac{\log(2/\delta)}{2T_{k^*}(\rho_{m}, r)}} -2 (T^{-1/2} \lor B^*)\leq 0.
    \]
    Therefore, $\tilde \Delta_{k^*}(\rho_{m}, r)= 0$ by its definition and this shows that no change point is detected by Algorithm~\ref{alg:select1}.  This contradicts with the fact that $r$ can take the value $\rho_{m+1}$, and therefore also the fact that $[\rho_m, \rho_{m+1}] \subset [\tau_m', \tau_{m'+1})$.  This concludes the proof. 
\end{proof}

\begin{corollary}\label{cor-bounding-m1}
Under Assumptions~\ref{As:2-sec-obs-gap} and \ref{As:1}, for any $\delta \in [0, 1]$, on the event $\xi(\delta)$, we have that $|\mathcal{M}_1| \leq 2M$.
\end{corollary}

\begin{proof}
This is a direct consequence of Proposition~\ref{prop:jump} and the definition of $\mathcal{M}_1$.
\end{proof}

\subsection{Low pulls of sub-optimal arms}  \label{sec-low-pulls}

\begin{proposition}\label{prop:suboptimalarm}
Under Assumptions~\ref{As:2-sec-obs-gap} and \ref{As:1}, let $u \in \mathcal{M}_1$.  For any $\delta \in [0, 1]$, on the event $\xi (\delta)$, the number of times a sub-optimal arm $k$ of gap 
    \begin{equation}\label{eq-prop3-cond}
        \bar \Delta_{k}(u) > 4(T^{-1/2} \lor B^*)
    \end{equation}
    is pulled on $I_u$ is upper bounded as
    \[
        T_k(i_{u}, i_{u+1}) \leq c/2 \log(2/\delta) \bar \Delta_{k}(u)^{-2},
    \]
for any constant $c \geq 16$.
\end{proposition}

\begin{proof} 
We prove by contradiction and assume that with $c \geq 16$,
    \[
        T_k(i_{u}, i_{u+1}) > c/2 \log(2/\delta) \bar \Delta_{k}(u)^{-2},
    \]
    which implies that
    \begin{align}\label{eq:cond}
        \bar\Delta_{k}(u) &\geq \sqrt{\frac{c\log(2/\delta)}{2T_k(i_{u}, r) }}.
    \end{align}
    On the event $\xi = \xi(\delta)$, we have that
    \begin{align}
        \hat \Delta_k(i_u, r)&\geq \frac{1}{T_k(i_{u}, r)}\sum_{j=i_u}^{r-1} \bar \Delta_k(u) \mathbf 1\{k \in K_j\}- \sqrt{\frac{\log(2/\delta)}{2T_k(i_u, r)}}. \label{Eq:1}
    \end{align}
    By definition, it holds that
    \begin{align*}
        & \tilde \Delta_k(i_u, r) = \hat \Delta_k(i_u, r) -\sqrt{\frac{\log(2/\delta)}{2T_k(i_u, r)}}-2 (T^{-1/2} \lor B^*) \\
        \geq & \frac{1}{T_k(i_{u}, r)}\sum_{j=i_u}^{r-1} \bar \Delta_k(u) \mathbf 1\{k \in K_j\} -  2\sqrt{\frac{\log(2/\delta)}{2T_k(i_u, r)}} - 2 (T^{-1/2} \lor B^*) \\
        = & \bar \Delta_k(u) - 2\sqrt{\frac{\log(2/\delta)}{2T_k(i_u, r)}} - 2 (T^{-1/2} \lor B^*) \geq (1 - 2c^{-1/2}) \bar{\Delta}_k(u) - 2 (T^{-1/2} \lor B^*) > 0,
    \end{align*}
    where the fist identity is due to the definition of $\tilde{\Delta}_k(\cdot, \cdot)$, the first inequality is due to \eqref{Eq:1}, the second identity is due to Assumption~\ref{As:2-sec-obs-gap}, the second inequality is due to \eqref{eq:cond}, and the last inequality is due to \eqref{eq-prop3-cond} and the fact that $c \geq 16$.  This means that arm $k$ should have been eliminated, which contradicts with the fact that $r \in I_u$.
\end{proof}

\subsection{Regret analysis}\label{Proof:th1:regret}

\subsubsection{Regret of the intervals in $\mathcal{M}_1$} \label{Proof:th1:regret1}

In this section we derive the regret only on the intervals of the partition $(I_u)_{u \in \mathcal{M}_1}$. Denote by $R'_T$ the contribution of this case to the overall regret $R^{\gap}_T$. We have that
    \begin{align}
        R'(T) &=\mathbb E \left[\sum_{u \in \mathcal{M}_1} \sum_{r \in I_u} \sum_{k\in K_r} \Delta_k( \bar t_r(k))\right]  \nonumber \\
        & \leq \mathbb E \left[\sum_{u \in \mathcal{M}_1} \sum_{k\in [K]} 2\bar \Delta_k(u) \left[T_k(i_u, i_{u+1})\right] \right] +  8T(B^{*} \lor T^{-1/2}).  \label{reg:eq1}
    \end{align}
    The second term in \eqref{reg:eq1} represents the regret of pulling an optimal arm, or sub-optimal arms satisfying \eqref{eq-prop3-cond} at every time point.  Now we split the rest of the arms by their gaps, i.e.
    \[
        \mathbb G_{1,u} = \left\{k\in [K]:\, 4(T^{-1/2} \lor B^*) \leq \bar \Delta_k(u) < \sqrt{\frac{K}{|I_u|}}\right\} 
    \]
    and
    \[
         \mathbb G_{2,u} = \left\{k\in [K]:\, \left\{4(T^{-1/2} \lor B^*) \lor \sqrt{\frac{K}{|I_u|}}\right\} \leq \bar \Delta_k(u) \right\}.
    \]
    Based on this, we have that
    \begin{align}
        R'(T) & \leq \mathbb E \Big[\sum_{u \in \mathcal{M}_1} \sum_{k\in K}2 \bar \Delta_k(u) T_k(i_u, i_{u+1})\Big] +  8T(B^{*} \lor T^{-1/2}) \nonumber\\ 
        & \leq \underbrace{ \mathbb E \Big[\sum_{u \in \mathcal{M}_1 }\sum_{k \in \mathbb G_{1,u}}2 T_k(i_u, i_{u+1}) \bar \Delta_k(u)\Big]}_{R_1(T)} + \underbrace{\mathbb E \Big[\sum_{u \in \mathcal{M}_1 } \sum_{k \in \mathbb G_{2,u}} 2 T_k (i_u, i_{u+1}) \bar \Delta_k(u) \Big]}_{R_2(T)} \nonumber \\
        & \hspace{2cm} +  \underbrace{8T(B^{*} \lor T^{-1/2})}_{R_3(T)}.\label{reg:eq3} 
    \end{align}

As for $R_1(T)$, by the definition of $\mathbb G_{1,u}$, we have that
    \begin{align}
        R_1(T) &= \mathbb E \Big[ \sum_{u \in \mathcal{M}_1 } \sum_{k \in \mathbb G_{1,u}} 2T_k(i_u,i_{u+1}) \bar \Delta_k(u)\Big] \leq \mathbb E \left[\sum_{u \in \mathcal{M}_1 } \sum_{k \in \mathbb G_{1,u}} 2\sqrt{\frac{K}{|I_u|}} T_k(i_u, i_{u+1})\right] \nonumber \\
        & \leq 2 \sqrt{K} \mathbb E \Big[\sum_{u \in \mathcal{M}_1 }\sqrt{|I_u|} \Big] \leq 2 \sqrt{ KT}\mathbb E [\sqrt{|\mathcal{M}_1}| ] \leq 2 \sqrt{KT} \sqrt{2M},\label{reg:eq6}
    \end{align}
    where the second inequality is due to the fact that $\sum_{k \in \mathbb G_{1,u}}  T_k(i_u, r)\leq |I_u|$ and the last inequality is due to Corollary~\ref{cor-bounding-m1}.
    
As for $R_2(T)$, taking $\delta = T^{-3}K^{-1}$, we have that
    \begin{align}
        R_2(T) &=\mathbb E \Big[\sum_{u \in \mathcal{M}_1 } \sum_{k \in \mathbb G_{2,u}} 2 T_k (i_u, i_{u+1}) \bar \Delta_k(u)\Big] \leq \mathbb E \Big[\sum_{u \in \mathcal{M}_1 }\sum_{k \in \mathbb G_{2,u}} c\log(2KT^3) \bar \Delta_k(u)^{-2} \bar \Delta_k(u)\Big] \nonumber\\
        &= \mathbb E \left[\sum_{u \in \mathcal{M}_1 }\sum_{k \in \mathbb G_{2,u}} c\log(2K T^3)\bar \Delta_k(u)^{-1}\right] \leq \mathbb E \left[\sum_{u \in \mathcal{M}_1}\sum_{k \in \mathbb G_{2,u}} \sqrt{\frac{|I_u|}{K}} c\log(2K T^3) \right]\nonumber\\
        & \leq \mathbb E \left[\sum_{u \in \mathcal{M}_1}c\log(2K T^3)\sqrt{|I_u|K}\right] \leq  c\log(2K T^3) \sqrt{KT}\mathbb E \Big[\sqrt{|\mathcal{M}_1|}] \nonumber\\
        & \leq 2^{1/2} c\log(2K T^3) \sqrt{KTM}, \label{reg:eq7}
    \end{align}
    where the first inequality follows from Proposition~\ref{prop:suboptimalarm}, the second inequality follows from the definition of $\mathbb G_{2,u}$, the last inequality follows from Corollary~\ref{cor-bounding-m1}.

Combining \eqref{reg:eq3}, \eqref{reg:eq6} and~\eqref{reg:eq7}, we have that with $c \geq 16$,
    \begin{align}\label{eq:12}
        R'(T) \leq \{2^{3/2} + 2^{1/2} c\log(2K T^3)\} \sqrt{KTM} + 8T(B^{*} \lor T^{-1/2}). 
    \end{align}

\subsubsection{Regret of the intervals in $\mathcal{M}_2$} \label{Proof:th1:regret2}
In this section we derive the regret only on the intervals of the partition $(I_u)_{u \in \mathcal{M}_2}$. Denote by $R''_T$ the contribution of this case to the overall regret $R^{\gap}_T$. For any $u\in \mathcal{M}_2$, there exists $m \in \{1, \ldots, M\}$, such that $I_u = [\tau_m -1, \tau_m)$.  This implies that $I_u$ contains exactly one round and therefore $T_k(i_u, i_{u+1}) \leq 1$.  We thus have that
    \begin{align}
        R''(T) &\leq \mathbb E \Big[\sum_{u \in \mathcal{M}_2 } \sum_{k\in K} 2\bar \Delta_k(u) \left[T_k(i_u, i_{u+1})\right] \Big] + 8T(B^* \lor T^{-1/2}) \nonumber\\
        &\leq \mathbb E \Big[ \sum_{u \in \mathcal{M}_2 } \sum_{k\in K} 2\bar \Delta_k(u) \Big] + 8T(B^* \lor T^{-1/2}) \leq 2 K\mathbb E [|\mathcal{M}_2| ] + 8T(B^* \lor T^{-1/2}) \nonumber\\
        & \leq 2 K M + 8T(B^* \lor T^{-1/2}), \label{eq-R''-bound}
    \end{align}
    where the third inequality is due to Assumption~\ref{As:1} and the last inequality is due to the definition of $\mathcal{M}_2$.

\subsubsection{Completion}

Combining \eqref{eq:12} and \eqref{eq-R''-bound}, we have that with $\delta = T^{-3}K^{-1}$, on the event $\xi(\delta)$, 
    \begin{align*}
       R^{\gap}_T \leq \{2^{3/2} + 2^{1/2} c\log(2K T^3)\} \sqrt{KTM} + 2KM + 8T(B^{*} \lor T^{-1/2}),
    \end{align*}
    for $c \geq 16$.

\section{Proof of Theorem~\ref{th:2}}\label{sec:1}

The proof is divided into several steps.  In Section~\ref{sec-proofb-refine}, we refine the partition proposed in Section~\ref{sec:IntnoChange}.  In Section~\ref{sec-proof-th2-large-prob-event}, we define a large probability event, in which the rest of the proof will be conducted.  In Section~\ref{sec-no-spurious-proof-th2}, it will be shown that Algorithm~\ref{alg:select} will not over-estimate the number of change points.  In Section~\ref{sec-low-pull-th2}, we control the numbers of pulling sub-optimal arms.  In Section~\ref{optevict}, we control the duration after all optimal arms are evicted.  Finally, we reach the conclusion in Section~\ref{sec-regret-final-th2}.

\subsection{Refinement of the partition}\label{sec-proofb-refine}

We will maintain throughout this proof the partition we introduced in Section~\ref{sec:IntnoChange}.  To be specific, we have that $(I_u)_{u\leq \tilde M}$ form a partition of $\{1,...,R\}$.  In order to prove Theorem~\ref{th:2}, we need to further refine the partition $(I_u)_{u \in \mathcal{M}_1}$ according to the different possible scenarios that may occur during the game.  According to the design of Algorithm~\ref{alg:select}, it may happen that, in a certain round $r_u \in I_u$, for $u\in \mathcal{M}_1$, all optimal arms $k^*_u\in [K]$ are evicted from the set of active arms.  Based on this consideration, for any $u \in \mathcal{M}_1$, we subdivide the interval $I_u$ into two intervals
    \[
        I'_{u,1}=\{r \in I_u : \,  \exists  k \in K_r, \, \mbox{s.t.} \, \bar{\Delta}_k(u) = 0 \}=[i_u, r_u)
    \]
    and
    \[
        I'_{u,2}=\{r \in I_u :\,  \forall k \in K_r, \, \bar{\Delta}_k(u) > 0 \}=[r_u, i_{u+1}),
    \]
    where $I'_{u,1}$ represents the set of rounds in interval $I_u$ containing at least one optimal arm and $I'_{u,2}$ consists of rounds without optimal arms.  Since $I_u = I'_{u,1} \sqcup I'_{u,2}$, $(I'_{u,1}, I'_{u,2})$ constitutes a partition of $I_u$.  We remark that it is not necessary that for any $u \in \mathcal{M}_1$, there exists such $r_u$, when all optimal arms are evicted.  In this case, we let $I_{u,1}' = [i_u, i_{u+1})$ and $I_{u, 2}' = \emptyset$.
    
We define the refinement $(I'_v)_{v\in \mathcal{C}}$ for the elements of the set $\{(I'_{u,1})_{u\in \mathcal{M}_1}, (I'_{u,2})_{u\in \mathcal{M}_1}\}$, reindexed in an increasing order.  To be specific, we let $v = 2(u-1) + 1$, if there exists $u \in \mathcal{M}_1$ such that $I'_v = I_{u,1}'$; and let $v = 2u$, if  there exists $u \in \mathcal{M}_1$ such that $I'_v = I_{u,2}'$.  The set $\mathcal{C}$ is define as $\mathcal{C}= \mathcal{C}_1 \cup \mathcal{C}_2$, where
    \[
        \mathcal{C}_1=\{v \in \mathcal{C}: \, I'_v=I'_{u,1}, \, u \in \mathcal{M}_1\} \quad \mbox{and} \quad \mathcal{C}_2=\{ v \in \mathcal{C}: \, I'_v=I'_{u,2}, \, u \in \mathcal{M}_1\}.
    \]
    Moreover, note that $(I'_v)_{v \in \mathcal{C}}$ is a refinement of the partition $(I_u)_{u\in \mathcal{M}_1}$ constituted of intervals with no changes, due to the definition of $\mathcal{M}_1$.  For any $ v \in \mathcal{C}$, we let $I'_v=[i'_v, i'_{v+1})$.

Furthermore, for any $v \in \mathcal{C}$, by the definition of $I'_v$, there exists $m'\in \{1, \ldots, M\}$ such that $I'_v \subset [\bar \tau_{m'}, \bar \tau_{m'+1})$.  Therefore, due to Assumption~\ref{As:2-sec-obs-gap1}, we let
    \[
        \bar \Delta_{k}'(v)=2^{\kappa_{k,m'}/2} (T^{-1/2} \lor B^*),
    \]
    as the lower bound of $\Delta_k(t)$ in all rounds in $I'_v$ - see Equation~\eqref{eq:assump} for the definition of $\kappa_{k,m'}$. From now on, we will stop using the $\bar \Delta_k(u)$ defined in Section~\ref{sec:IntnoChange} for $u \in \mathcal M_1$, and use instead $\bar \Delta_{k}'(v)$ for $v\in \mathcal C$. \textit{For this reason and from now on, in order to alleviate notation, we will use $\bar \Delta_k(v)$ instead of $\bar \Delta'_k(v)$, since no confusion can arise as we will not use $\bar \Delta_k(u)$ anymore.} For any $v \in \mathcal{C}$, we define the gap of an arm $k \in [K]$ with respect to arm $k'$ in interval $I'_v$ by
    \[
        G_{k}^{(k')}(v) = \Big(\bar \Delta_{k}(v)/2- 2\bar \Delta_{k'}(v)\Big) \lor 0, \quad \forall k, k'\in [K].
    \]
 
To continue, for any $v \in \mathcal{C}$ and $m\leq M_R$ satisfying that $I'_v \subset [\rho_{m}, \rho_{m+1})$, we denote $\hat k_v$ as one of the arms in $\mathcal S(i'_{v}, i'_{v+1})$.  Such an arm must exist by Algorithm~\ref{alg:select}.
 
\subsection{The high probability event} \label{sec-proof-th2-large-prob-event}

For any $\delta \in [0, 1]$, define $\xi(\delta)$ as
    \begin{align*}
        & \xi(\delta) = \Bigg\{\left|\hat \Delta_k^{(k')}(r',r) - \frac{1}{T_k(r',r)} \sum_{j=r'}^{r-1} \{\Delta_{k}(\bar t_j(k)) - \Delta_{k'}(\bar t_j(k'))\} \mathbf 1\{k \in K_j\} \right|   \\
        & \leq \sqrt{\frac{2\log(2/\delta)}{T_k(r',r)}} + B^{*}, \quad \forall k \in [K], \, T_k(r', r) > 0 \, \forall 1 \leq r' < r \leq R, \, \forall k' \in \mathcal S(r',r)\Bigg\}.    
    \end{align*}
 	
\begin{proposition}\label{Prop:4}
    Under Assumptions~\ref{As:1-unobs1}, \ref{As:2-unobs1} and \ref{As:2-sec-obs-gap1}, for any $\delta \in [0, 1]$, the event $\xi(\delta)$ holds with probability larger than $1-K T^2\delta$.
\end{proposition}

\begin{proof}
For any $k\in [K]$, any $1 \leq r' < r \leq R$ and $k' \in \mathcal S(r',r)$, by the definition of $\hat \Delta_k^{(k')}(r',r)$, we have that 
    \begin{align}
        & \left|\hat \Delta_k^{(k')}(r',r) - \frac{1}{T_k(r',r)} \sum_{j=r'}^{r-1} \{\Delta_{k}(\bar t_j(k)) - \Delta_{k'}(\bar t_j(k'))\} \mathbf 1\{k \in K_j\} \right| \nonumber\\
        = & \frac{1}{T_k(r',r)}\left|\sum_{j=r'}^{r-1}\{\tilde X_j(k')-\tilde X_j(k)- \Delta_{k}(\bar t_j(k))+\Delta_{k'}(\bar t_j(k'))\}\mathbf 1\{k \in K_j\}\right| \nonumber \\
        = & \frac{1}{T_k(r',r)} \left|\sum_{j=r'}^{r-1}\{\epsilon_{\bar t_j(k')}^{(k')}-\epsilon_{\bar t_j(k)}^{(k)} + \mu^*(\bar t_j(k')) - \mu^*(\bar t_j(k)) \} \mathbf 1\{k \in K_j\}\right| \nonumber\\
        \leq & \frac{1}{T_k(r',r)} \left|\sum_{j=r'}^{r-1}\{\epsilon_{\bar t_j(k')}^{(k')} - \epsilon_{\bar t_j(k)}^{(k)}\} \mathbf 1\{k \in K_j\}\right| + \frac{1}{T_k(r',r)}\sum_{j=r'}^{r-1} |\mu^*(\bar t_j(k')) - \mu^*(\bar t_j(k))| \mathbf 1\{k \in K_j\} \nonumber\\
        \leq & \frac{1}{T_k(r',r)}\left|\sum_{j=r'}^{r-1}\{\epsilon_{\bar t_j(k')}^{(k')}-\epsilon_{\bar t_j(k)}^{(k)}\} \mathbf 1\{k \in K_j\}\right| + B^*, \nonumber
    \end{align}
    where the last inequality is due to Assumption~\ref{As:2-sec-obs-gap}, noting that $|\bar t_j(k')- \bar t_j(k)|\leq K$ and there is $m'\leq M_R$ such that $\bar t_j(k), \bar t_j(k') \in C_{m'}$. 

In addition, it follows from the same arguments as those in the proof of Proposition~\ref{prop:clean_event}, we complete the proof.
\end{proof}

\subsection{No spurious jump detected} \label{sec-no-spurious-proof-th2}

\begin{proposition}\label{prop5}
Under Assumptions~\ref{As:1-unobs1}, \ref{As:2-unobs1} and \ref{As:2-sec-obs-gap1}, for any $m \in \{1, \ldots, M_R\}$ and any $\delta \in [0, 1]$, we have that on the event $\xi(\delta)$, $[\rho_{m}, \rho_{m+1}]$ is not included in any connex component $[\bar \tau_{m'}, \bar \tau_{m'+1})$, $m'\leq M$.
\end{proposition}

\begin{proof} 
For notational simplicity, we let $\xi = \xi(\delta)$.  The rest of the proof is conducted on the event $\xi$.  We prove by contradiction and assume that for a fixed $m \in \{1, \ldots, M_R\}$, there exists $m'\leq M$ such that $[\rho_{m}, \rho_{m+1}] \subset [\bar \tau_{m'}, \bar \tau_{m'+1})$.  Then by the constitution of the partition $(I_u)_{u \in \mathcal{M}_1}$, there exists $u \in \mathcal{M}_1$ such that $I_u=[\rho_{m}, \rho_{m+1}]$.  

\medskip
\noindent \textbf{Step 1.}  First we note that, it follows from Assumption~\ref{As:2-sec-obs-gap1} that there exists an optimal arm $k^*$ such that $\Delta_{k^*}(u) = 0$.  Therefore, for any $\rho_m < r \leq \rho_{m+1}$ and $k' \in \mathcal{S}(\rho_m, r)$, on the event $\xi$, it holds that
    \begin{align}
        \Delta_{k^*}^{(k')}(\rho_m,r) &\leq \frac{1}{T_{k^*}(\rho_m,r)} \sum_{j=\rho_m}^{r-1} \{\Delta_{k^*}(\bar t_j(k^*)) - \Delta_{k}(\bar t_j(k'))\} \mathbf 1\{k \in K_j\}+\sqrt{\frac{2\log(2/\delta)}{T_{k^*}(\rho_m,r)}} +B^*\nonumber\\
        & \leq \frac{1}{T_{k^*}(\rho_m,r)}\sum_{j=\rho_m}^{r-1}\{\mu^*(\bar t_j(k^*))- \mu_{k^*}(\bar t_j(k^*))\}\mathbf 1\{k^* \in K_j\} \nonumber \\ 
        & \hspace{1cm} -\frac{1}{T_{k^*}(\rho_m,r)}\sum_{j=\rho_m}^{r-1}\{\mu^*(\bar t_j(k'))- \mu_{k'}(\bar t_j(k')\} \mathbf 1\{k^* \in K_j\}+\sqrt{\frac{2\log(2/\delta)}{T_{k^*}(\rho_m,r)}} +B^*\nonumber\\
        & \leq \frac{1}{T_{k^*}(\rho_m,r)}\sum_{j=\rho_m}^{r-1}\{\mu^*(\bar t_j(k^*))- \mu_{k^*}(\bar t_j(k^*))\}\mathbf 1\{k^* \in K_j\} +\sqrt{\frac{2\log(2/\delta)}{T_{k^*}(\rho_m,r)}} +B^*\nonumber\\
        & \leq \sqrt{\frac{2\log(2/\delta)}{T_{k^*}(\rho_m,r)}} + 2B^*, \nonumber
    \end{align}
    where the last inequality follows from Assumption~\ref{As:2-unobs1} that
    \[
        |\mu^*(\bar t_j(k^*))- \mu_{k^*}(\bar t_j(k^*))| \leq B^*,
    \]
    and the fact that $\mu^*(\bar t_j(k')) \geq \mu_{k'}(\bar t_j(k')$.  
    
Therefore, it follows from the definition of $\tilde{\Delta}_k(r', r)$ in \eqref{eq-tilde-Delta-definiton}, we have that $\tilde{\Delta}_{k^*}(\rho_m, r) = 0$, which implies that $\tilde{N}_{k^*}(r) = 0$.  According to the design of Algorithm~\ref{alg:select}, we know that $k^* \in \mathcal{S}(\rho_m, \rho_{m+1})$.

\medskip
\noindent \textbf{Step 2.}  It then follows from the similar as in \textbf{Step 1} that, for any $r \in (\rho_m, \rho_{m+1}]$,
    \[
        |\hat{\Delta}_{k^*}(\rho_m, r) - \hat{\Delta}_{k^*}(r, \rho_{m+1})| \leq 2\sqrt{\frac{2\log(2/\delta)}{T_{k^*}(\rho_m,r) \wedge T_{k^*}(r, \rho_{m+1}) }} + 2B^*,
    \]
    which suggests that arm $k^*$ will not trigger a change point, according to Algorithm~\ref{alg:changepoint}.
    
As for any other arm $k \neq k^*$, we have that for any $r \in (\rho_m, \rho_{m+1})$, on the event $\xi_1$, it holds that 
    \begin{align*}
        & \hat{\Delta}_k(\rho_m, r) - \hat{\Delta}_k(r, \rho_{m+1}) = \max_{k' \in \mathcal{S}(\rho_m, r)} \hat{\Delta}^{(k')}_k (\rho_m, r) - \max_{k' \in \mathcal{S}(r, \rho_{m+1})} \hat{\Delta}^{(k')}_k \\
        \leq & \frac{1}{T_k(\rho_m, r)} \sum_{j = \rho_m}^{r-1} \{\mu^*(\bar{t}_j(k)) - \mu_k(\bar{t}_j(k)) - \mu^*(\bar{t}_j(k^*)) + \mu_{k^*}(\bar{t}_j(k^*))\}  \mathbbm{1}\{k \in K_j\} \\
        & \hspace{1cm} - \frac{1}{T_k(r, \rho_{m+1})} \sum_{j = r}^{\rho_{m+1} - 1} \{\mu^*(\bar{t}_j(k)) - \mu_k(\bar{t}_j(k)) - \mu^*(\bar{t}_j(k^*)) + \mu_{k^*}(t_j(k^*))\} \mathbbm{1}\{k \in K_j\} \\
       & \hspace{1cm} + 2\sqrt{\frac{2\log(2/\delta)}{T_k(\rho_m, r) \wedge T_k(r, \rho_{m+1})}} + 2B^* \\
       \leq & \frac{1}{T_k(\rho_m, r)} \sum_{j = \rho_m}^{r-1} \{\mu^*(\bar{t}_j(k)) - \mu_k(\bar{t}_j(k))\}  \mathbbm{1}\{k \in K_j\} \\
        & \hspace{1cm} - \frac{1}{T_k(r, \rho_{m+1})} \sum_{j =r}^{\rho_{m+1} - 1} \{\mu^*(\bar{t}_j(k)) - \mu_k(\bar{t}_j(k))\}\mathbbm{1}\{k \in K_j\} \\
        & \hspace{1cm} + 2\sqrt{\frac{2\log(2/\delta)}{T_k(\rho_m, r) \wedge T_k(r, \rho_{m+1})}} + 2B^*,
    \end{align*}
    where the first inequality follows from the definition of the event $\xi$.
    
Similar arguments lead to the following
    \begin{align*}
        & \hat{\Delta}_k(\rho_m, r) - \hat{\Delta}_k(r, \rho_{m+1}) \\
        \geq & \frac{1}{T_k(\rho_m, r)} \sum_{j = \rho_m}^{r-1} \{\mu^*(\bar{t}_j(k^*)) - \mu_{k}(\bar{t}_j(k))\} \mathbbm{1}\{k \in K_j\} \\
        & \hspace{1cm} - \frac{1}{T_k(r, \rho_{m+1})} \sum_{j = r}^{\rho_{m+1}-1} \{\mu^*(\bar{t}_j(k)) - \mu_k(\bar{t}_j(k))\} \mathbbm{1}\{k \in K_j\} \\
        & \hspace{1cm} - 2\sqrt{\frac{2\log(2/\delta)}{T_k(\rho_m, r) \wedge T_k(r, \rho_{m+1})}} - 2B^*.
    \end{align*}
    
Then it holds that 
    \begin{align}
        & |\hat{\Delta}_k(\rho_m, r) - \hat{\Delta}_k(r, \rho_{m+1})| \nonumber \\
        \leq & \Bigg|\frac{1}{T_k(\rho_m, r)} \sum_{j = \rho_1}^{r-1} \{\mu^*(\bar{t}_j(k)) - \mu_k(\bar{t}_j(k))\} \mathbbm{1}\{k \in K_j\} \nonumber \\
        & \hspace{1cm} - \frac{1}{T_k(r, \rho_{m+1})} \sum_{j = r}^{\rho_{m+1}-1} \{\mu^*(\bar{t}_j(k)) - \mu_k(\bar{t}_j(k))\} \mathbbm{1}\{k \in K_j\}\Bigg| \nonumber \\
        & \hspace{1cm} + 2\sqrt{\frac{2\log(2/\delta)}{T_k(\rho_m, r) \wedge T_k(r, \rho_{m+1})}} + 2B^* \nonumber \\
        \leq & 2^{-1} \max\Bigg\{\frac{1}{T_k(\rho_m, r)} \sum_{j = \rho_1}^{r-1} \{\mu^*(\bar{t}_j(k)) - \mu_k(\bar{t}_j(k))\} \mathbbm{1}\{k \in K_j\} , \nonumber \\
        & \hspace{1cm} \frac{1}{T_k(r, \rho_{m+1})} \sum_{j = r}^{\rho_{m+1}-1} \{\mu^*(\bar{t}_j(k)) - \mu_k(\bar{t}_j(k))\} \mathbbm{1}\{k \in K_j\}\Bigg\} \nonumber \\
        & \hspace{1cm} + 2\sqrt{\frac{2\log(2/\delta)}{T_k(\rho_m, r) \wedge T_k(r, \rho_{m+1})}} + 2B^*, \label{eq-delta-delta-upper-bound}
    \end{align}
where the last inequality follows from Assumption~\ref{As:2-sec-obs-gap1}.

On the other hand, we have that
    \begin{align}
        & \max \left\{\hat{\Delta}_k(\rho_m, r), \, \hat{\Delta}_{k}(r, \rho_{m+1})\right\} \nonumber \\
        \geq & \max\Bigg\{\frac{1}{T_k(\rho_m, r)} \sum_{j = \rho_1}^{r-1} \{\mu^*(\bar{t}_j(k)) - \mu_k(\bar{t}_j(k))\} \mathbbm{1}\{k \in K_j\}, \nonumber \\
        & \hspace{1cm} \frac{1}{T_k(r, \rho_{m+1})} \sum_{j = r}^{\rho_{m+1}-1} \{\mu^*(\bar{t}_j(k)) - \mu_k(\bar{t}_j(k))\} \mathbbm{1}\{k \in K_j\}\Bigg\} \nonumber  \\
        & \hspace{1cm}  - 2\sqrt{\frac{2\log(2/\delta)}{T_k(\rho_m, r) \wedge T_k(r, \rho_{m+1})}} - 2B^*. \label{eq-max-delta-delta-lower-bound}
    \end{align}
    
Combining \eqref{eq-delta-delta-upper-bound} and \eqref{eq-max-delta-delta-lower-bound}, we have that 
    \begin{align*}
        & |\hat{\Delta}_k(\rho_m, r) - \hat{\Delta}_k(r, \rho_{m+1})| \leq 2^{-1} \max\left\{\hat{\Delta}_k(\rho_m, r), \, \hat{\Delta}_{k}(r, \rho_{m+1})\right\} \\
        & \hspace{1cm} +3\sqrt{\frac{2\log(2/\delta)}{T_k(\rho_m, r) \wedge T_k(r, \rho_{m+1})}} + 3B^*,
    \end{align*}
    which implies that $\rho_{m+1}$ is not a detected change point based on the \eqref{Chp:eq1}.
\end{proof}

\begin{corollary}\label{cor2}
Under Assumptions~\ref{As:1-unobs}, \ref{As:2-unobs} and \ref{As:2-sec-obs-gap}, for any $\delta \in [0, 1]$, on the event $\xi(\delta)$, we have that 
    \[
        |\mathcal{C}_2| \leq |\mathcal{C}| \leq |\mathcal{M}_1| \leq 2M.
    \]
\end{corollary}

\begin{proof}
This is a direct consequence of Proposition~\ref{prop5}.
\end{proof}

\subsection{Low pulls of sub-optimal arms}\label{sec-low-pull-th2}

In this subsection, we upper bound the number of times a sub-optimal arm is pulled in the intervals $I'_{v}=[i'_v, i'_{v+1})$, $ v \in \mathcal{C}$.  There are two cases: either the gap of the arm in question with respect to $\hat k_v$ has been over-estimated or under-estimated.

\subsubsection{Over-estimated}

\begin{proposition}\label{prop:6}
Under Assumptions~\ref{As:1-unobs}, \ref{As:2-unobs} and \ref{As:2-sec-obs-gap}, let $k$ be a sub-optimal arm with respect to $\hat k_v$, $v \in \mathcal{C}$.  Let $m'\leq M$ satisfy that $I'_v \subset [\rho_{m'}, \rho_{m'+1})$.  If 
    \begin{equation}\label{eq-tilde-n-lower-bound}
        \tilde N_k(i'_{v+1} - 1) \geq G_k^{(\hat k_v)}(v) \sqrt{\frac{TK}{M}},
    \end{equation}
    and
    \begin{equation}\label{eq-g-kv-k-v-lb-cond}
        G^{(\hat k_v)}_{k}(v) > 4B^*,        
    \end{equation}
    then for any $\delta \in [0, 1]$, on the event $\xi(\delta)$, for an absolute constant $c \geq 1$, it holds that
    \[
        T_k(i'_{v}, i'_{v+1}) \leq \frac{8 \log(2/\delta)}{\{G^{(\hat{k}_v)}_k(v)\}^2} + \frac{|I'_v|\sqrt{M}}{\sqrt{TK}} \frac{1}{ G_{k}^{(\hat k_v)}(v)} + 1.
    \]
\end{proposition}

\begin{proof}
Since $\mathcal{S}(i'_v, i'_{v+1}) \neq \emptyset$, there exists an arm $\hat k_v$ such that $\hat k_v \in \mathcal{S}(i'_v, i'_{v+1})$.   On the event $\xi(\delta)$, for any $r \in (i'_v, i'_{v+1})$, it holds with probability larger than $1-KT^2\delta$ that 
    \begin{align}
        \hat \Delta_{k}^{(\hat k_v)}(i'_v, r) &\geq \frac{1}{T_{k}(i'_v,r)} \sum_{j=i'_v}^{r-1} \{\Delta_{k}(\bar t_j(k)) - \Delta_{\hat{k}_v}(\bar t_j(\hat k_v))\} \mathbf 1\{k \in K_j\}-\sqrt{\frac{2\log(2/\delta)}{T_{k}(i'_v,r)}} -B^* \nonumber \\
        & \geq \frac{1}{T_{k}(i'_v,r)} \sum_{j=i'_v}^{r-1}\{\bar \Delta_k(v) - 4\bar \Delta_{\hat k_v}(v)\} \mathbf 1\{k \in K_j\}-\sqrt{\frac{2\log(2/\delta)}{T_{k}(i'_v,r)}} -2 B^* \nonumber \\
        & = 2G^{(\hat k_v)}_{k}(v) - \sqrt{\frac{2\log(2/\delta)}{T_{k}(i'_v,r)}} - 2B^*,  \label{eq-hat-delta-kv-k-lb} 
    \end{align}
    where the second inequality follows from Assumption~\ref{As:2-sec-obs-gap}.
    
\medskip
\noindent \textbf{Step 1}  On the event $\xi(\delta)$, if $\tilde{\Delta}_k(i'_v, r) = 0$, then due to \eqref{eq-tilde-Delta-definiton}, we have that 
    \[
        \hat{\Delta}_k(i'_v, r) \leq \sqrt{\frac{2\log(2KT^2)}{T_k(i_v', r)}} +  2B^*,
    \]
    which combines with \eqref{eq-hat-delta-kv-k-lb} showing that, with $\delta = (KT^2)^{-1}$,
    \[
        2\sqrt{\frac{2\log(2/\delta)}{T_{k}(i'_v,r)}} \geq 2G^{(\hat k_v)}_{k}(v) - 4B^* \geq G^{(\hat{k}_v)}_k(v),
    \]
    where the last inequality follows from \eqref{eq-g-kv-k-v-lb-cond}.
    
Therefore, we have that
    \begin{equation}\label{eq-tk-iv-r11}
        T_k(i'_v, r) \leq \frac{8 \log(2/\delta)}{\{G^{(\hat{k}_v)}_k(v)\}^2}.
    \end{equation}
    
\medskip
\noindent \textbf{Step 2}  If there exists $r$ such that $\tilde{\Delta}_k(i'_v, r) > 0$, then, with some abuse of notation, we let $r$ be the round that arm $k$ is evicted.  Due to \eqref{eq-tilde-n-lower-bound} and the definition of $\tilde{N}_k(\cdot)$ in \eqref{eq-n-tilde-definition}, we have that
    \begin{equation}\label{eq-tk-iv-r22}
        T_k(r, i'_{v+1}) \leq \frac{|I'_v|}{\tilde{N}_k(i'_{v+1} - 1)} + 1 \leq \frac{|I_v'|}{G^{(\hat{k}_v)}_k(v)} \sqrt{\frac{M}{TK}} + 1.
    \end{equation}

Combining \eqref{eq-tk-iv-r11} and \eqref{eq-tk-iv-r22}, we conclude the proof.
\end{proof}

\subsubsection{Under-estimated}

\begin{proposition}\label{prop:7}
Under Assumptions~\ref{As:1-unobs}, \ref{As:2-unobs} and \ref{As:2-sec-obs-gap}, let $k$ be a sub-optimal arm with respect to $\hat k_v$, $v \in \mathcal{C}$.  Let $m'\leq M$ satisfy that $I_v' \subset [\rho_{m'}, \rho_{m'+1})$.  If 
    \begin{equation}\label{eq-tilde-N-upper-bound-cond}
        \tilde N_k(i'_{v+1}-1) \leq G_k^{(\hat k_v)} \sqrt{\frac{TK}{M}}
    \end{equation}
    and
    \begin{equation}\label{eq-gk-lb}
        G^{(\hat{k}_v)}_k > 4B^*,
    \end{equation}
    then for any $\delta \in [0, 1]$, on the event $\xi(\delta)$, it holds that
    \[
        T_k(i'_{v}, i'_{v+1}) \leq 12 \log(2/\delta) \left(G_{k}^{(\hat k_v)}\right)^{-2}.
    \]
\end{proposition}

\begin{proof}
This proof consists of two cases.

\medskip
\noindent \textbf{Case 1.}  If $\tilde{N}_k(i'_{v+1}-1) = 0$, then it follows from the identical arguments as those leading to \eqref{eq-hat-delta-kv-k-lb}, we have that for any $r \in I'_v$, 
    \[
        2G^{(\hat{k}_v)}_k(v) - \sqrt{\frac{2\log(2/\delta)}{T_k(i'_v, r)}} - 2B^* \leq \hat{\Delta}^{(\hat{k}_v)}_k(i'_v, r) \leq \sqrt{\frac{2\log(2/\delta)}{T_k(i'_v, r)}} + 2B^*.
    \]
    This leads to that
    \[
        2 \sqrt{\frac{2\log(2/\delta)}{T_k(i'_v, r)}} \geq 2G^{(\hat{k}_v)}_k - 4B^* > G^{(\hat{k}_v)}_k,
    \]
    where the last inequality follows from \eqref{eq-gk-lb}.  Then we have that
    \[
        T_k(i'_{v}, r) \leq 10 \log(2/\delta) \left(G_{k}^{(\hat k_v)}(v)\right)^{-2}.
    \]

\medskip
\noindent \textbf{Case 2.}  If $\tilde{N}_k(i'_{v+1}-1) > 0$, then due to the definition of $\tilde{N}_k(\cdot)$ and \eqref{eq-tilde-N-upper-bound-cond}, we have that there exists $r \in I'_v$ such that 
    \begin{equation}\label{eq-hat-delta-upper-bbbb}
        \hat{\Delta}_k(\rho_{m'}, r) \leq G_k^{(\hat k_v)}(v).
    \end{equation}
    Then we have that 
    \[
        \tilde{\Delta}_k(\rho_{m'}, r) \leq G_k^{(\hat k_v)}(v) - \sqrt{\frac{2\log(2/\delta)}{T_k(\rho_{m'}, r)}} - 2B^* \leq G_k^{(\hat k_v)}(v) - \sqrt{\frac{2\log(2/\delta)}{T_k(\rho_{m'}, r)}},
    \]
    which means that  we would detect a change point after sampling $2\log(2/\delta) (G_k^{(\hat k_v)}(v))^2$ times.
    
Combining the above two cases, we conclude the proof.    
\end{proof}

\subsection{Bound on the duration of the episode after all optimal arms are evicted} \label{optevict}

In this subsection, we are to determine the length of the intervals $I'_v$, for $v \in \mathcal{C}_2$. In other terms, we want to know how long the game can last after for all optimal arms are expelled from $\mathcal{S}(\cdot, \cdot)$.

\begin{proposition}\label{prop:8}
Under Assumptions~\ref{As:1-unobs}, \ref{As:2-unobs} and \ref{As:2-sec-obs-gap}, for any $v \in \mathcal{C}_2$ and any $\delta \in [0, 1]$, on the event $\xi(\delta)$, it holds that
    \[
        | i'_{v+1} - i'_{v}| \leq 25 \log(2/\delta) \bar \Delta_{\hat k_v}(v)^{-1} \sqrt{\frac{TK}{M}}.
    \]
\end{proposition}

\begin{proof}
This proof is conducted on the event $\xi(\delta)$.  Let $v\in \mathcal{C}_2$, then there exists $u \in \mathcal{M}_1$ and $m'\leq M_R$ such that $I'_v \subset I_u \subset [\rho_{m'}, \rho_{m'+1})$.  By the definition of $I'_v$, for all arms $k^*_v \in [K]$ satisfying that $\bar \Delta_{k^*_v}(v) = 0$, $k^*_v \not\in \mathcal{S}(i'_v, i'_{v+1})$.  This implies that $\tilde N_{k^*_v}(i'_{v}) > 0$.  Therefore, there exists $r' \in (\rho_{m'}, i'_{v}]$ such that $\tilde \Delta_{k^*_v}(\rho_{m'}, r') > 0$, since
    \begin{align} \label{prop8:eq1}
        \hat \Delta_{k^*_u}(\rho_{m'}, r')  &\geq \sqrt{\frac{2\log(2/\delta)}{T_{k^*_u}(\rho_{m'}, r') }} + 2B^* \geq \sqrt{\frac{2\log(2/\delta)}{T_{k^*_u}(\rho_{m'}, r') }} 
    \end{align}
    and $\tilde N_{k^*_u}(i'_{v})=\hat \Delta_{k^*_v}(\rho_{m'}, r') \sqrt{\frac{TK}{M}}$.
    
Since $I_u$ is an interval without change points, due to \eqref{Chp:eq1}, we have for $i'_v < r'' < i'_{v+1}$ and $r'< i'_v$ that
    \[
        |\hat \Delta_{k^*_u}(i'_v,r'') - \hat \Delta_{k^*_u} (\rho_{m'}, r')|\leq 2 \hat \Delta_{k^*_u}(i'_v,r'') + 2\sqrt{\frac{2\log(2/\delta)}{T_{k^*_u}(i'_v,r'') \land T_{k^*_u}(\rho_{m'}, r')}} + 2B^*,
    \]
    which implies that
\begin{align}
 3\hat \Delta_{k^*_u}(i'_v,r'') &\geq  \hat \Delta_{k^*_u} (\rho_{m'}, r') - 2 \sqrt{\frac{2\log(2/\delta)}{T_{k^*_u}(i'_v,r'') \land T_{k^*_u}(\rho_{m'}, r')}}- 2B^*. \label{prop8:eq2}
\end{align}

Moreover, we have that
    \[
        -2 \bar \Delta_{\hat k_u}(v) -  \sqrt{\frac{2\log(2/\delta)}{T_{k^*_u}(i'_v,r'')}} -B^* \leq \hat \Delta_{k^*_u}^{\hat k_v}(i'_v,r'') \leq - \bar \Delta_{\hat k_u}(v) +\sqrt{\frac{2\log(2/\delta)}{T_{k^*_u}(i'_v,r'')}} + B^*,
    \]
    where $\hat{k}_v$ satisfies that $\hat \Delta_{k^*_u}^{\hat k_v}(i'_v,r'')=\hat \Delta_{k^*_u}(i'_v,r'')$.  Thus, we have that
    \begin{align} \label{prop8:eq3}
        \hat \Delta_{k^*_u}(i'_v,r'') \leq - \bar \Delta_{\hat k_v}(v) + \sqrt{\frac{2\log(2/\delta)}{T_{k^*_u}(i'_v,r'')}} + B^*.
    \end{align}
 
Combining \eqref{prop8:eq2} and \eqref{prop8:eq3}, it holds that
    \begin{align}
        \hat \Delta_{k^*_u} (\rho_{m'}, r') - 2 \sqrt{\frac{2\log(2/\delta)}{T_{k^*_u}(i'_v,r'') \land T_{k^*_u}(\rho_{m'}, r')}}- 2B^*&\leq - 3\bar \Delta_{\hat k_v}(v) +3\sqrt{\frac{2\log(2/\delta)}{T_{k^*_u}(i'_v,r'')}} +3 B^*, \nonumber
    \end{align}
    which implies that
    \[
        \hat \Delta_{k^*_u} (\rho_{m'}, r') + 3\bar \Delta_{\hat k_v}(v) - 5B^* \leq 2 \sqrt{\frac{2\log(2/\delta)}{T_{k^*_u}(i'_v,r'') \land T_{k^*_u}(\rho_{m'}, r')}} + 3\sqrt{\frac{2\log(2/\delta)}{T_{k^*_u}(i'_v,r'')}}.
    \]
    By the definition of $\bar \Delta_{\hat k_v}(v)$, we have that $\bar \Delta_{\hat k_v}(v)\geq 2B^*$ which first implies that $3\bar \Delta_{\hat k_v}(v)\geq 6B^* \geq 5B^*$.  We therefore have that
    \[
        \hat \Delta_{k^*_u} (\rho_{m'}, r') \leq 2 \sqrt{\frac{2\log(2/\delta)}{T_{k^*_u}(i'_v,r'') \land T_{k^*_u}(\rho_{m'}, r')}} + 3\sqrt{\frac{2\log(2/\delta)}{T_{k^*_u}(i'_v,r'')}}. 
    \]
 
Since $5\bar \Delta_{\hat k_v}(v)\geq 10B^*$, it holds that 
    \begin{align}
        \hat \Delta_{k^*_u} (\rho_{m'}, r') + (1/2) \bar \Delta_{\hat k_v}(v)\leq 2 \sqrt{\frac{2\log(2/\delta)}{T_{k^*_u}(i'_v,r'') \land T_{k^*_u}(\rho_{m'}, r')}} + 3\sqrt{\frac{2\log(2/\delta)}{T_{k^*_u}(i'_v,r'')}} \nonumber
    \end{align}
    and 
    \[
        (1/2) \bar \Delta_{\hat k_v}(v) \leq2 \sqrt{\frac{2\log(2/\delta)}{T_{k^*_u}(i'_v,r'') \land T_{k^*_u}(\rho_{m'}, r')}} + 3\sqrt{\frac{2\log(2/\delta)}{T_{k^*_u}(i'_v,r'')}}, 
    \]
    Combining the above two, we deduce that
    \begin{align}
        \hat \Delta_{k^*_u} (\rho_{m'}, r') \lor (1/2)\bar \Delta_{\hat k_v}(v) &\leq 2 \sqrt{\frac{2\log(2/\delta)}{T_{k^*_u}(i'_v,r'') \land T_{k^*_u}(\rho_{m'}, r')}} + 3\sqrt{\frac{2\log(2/\delta)}{T_{k^*_u}(i'_v,r'')}} \label{prop8:eq5}
    \end{align}
 
\medskip
\noindent \textbf{Case 1.}  If $T_{k^*_u}(i'_v,r'') \leq T_{k^*_u}(\rho_{m'}, r')$, we have from \eqref{prop8:eq5} that
     \begin{align}
          (\hat \Delta_{k^*_u} (\rho_{m'}, r') \lor (1/2) \bar \Delta_{\hat k_v}(v) \leq 5 \sqrt{\frac{2\log(2/\delta)}{T_{k^*_u}(i'_v,r'')}}. \nonumber
    \end{align}
    Then,
    \[
        T_{k^*_u}(i'_v,r'') \leq 50 \log(2/\delta) \{\hat \Delta_{k^*_u} (\rho_{m'}, r')^{-2} \land (1/2)\bar \Delta_{\hat k_v}(v)^{-2}\}. 
    \]
    Without loss of generality that $r''=i'_{v+1}-1$, then $T_{k^*_u}(i'_v,r'')$ will represent the duration of the rest of the episode after all optimal arms $k^*_u$ have been eliminated.  This means that
    \begin{align}
        \frac{|i'_{v+1} - i'_v|}{\tilde N_{k^*_u}(i'_v)} &\leq 50 \log(2/\delta)\{\hat \Delta_{k^*_u} (\rho_{m'}, r')^{-2} \land (1/2) \bar \Delta_{\hat k_v}(v)^{-2}\}.\nonumber
    \end{align}
    Thus by the definition of $\tilde N_{k^*_u}(i'_v)$, we deduce that
    \begin{align}
        & |i'_{v+1} - i'_v| \leq 50\log(2/\delta)\{\hat \Delta_{k^*_u} (\rho_{m'}, r')^{-2} \land (1/2)\bar \Delta_{\hat k_u}(v)^{-2}\} \tilde N_{k^*_u}(i'_v) \nonumber \\
        \leq & 50 \log(2/\delta) \left\{\hat \Delta_{k^*_u} (\rho_{m'}, r')^{-1} \land \frac{\hat \Delta_{k^*_u} (\rho_{m'}, r')}{2\bar \Delta_{\hat k_v}(v)^{2}} \right\} \sqrt{\frac{TK}{M}} \nonumber\\
        \leq & 50 \log(2/\delta) \{\hat \Delta_{k^*_u} (\rho_{m'}, r')^{-1} \land  (1/2)\bar \Delta_{\hat k_v}(v)^{-1} \} \sqrt{\frac{TK}{M}} \leq 25 \log(2/\delta)  \bar \Delta_{\hat k_v}(v)^{-1}  \sqrt{\frac{TK}{M}}.\nonumber
    \end{align}

\medskip
\noindent \textbf{Case 2.}  If $T_{k^*_u}(\rho_{m'}, r') \leq T_{k^*_u}(i'_v,r'')$, then by \eqref{prop8:eq1} it holds that
    \begin{align}
        \hat \Delta_{k^*_u} (\rho_{m'}, r') & \geq  \sqrt{\frac{2\log(2/\delta)}{T_{k^*_u}(\rho_{m'}, r')}}.
    \end{align}
    It follows that
    \begin{align}
        T_{k^*_u}(\rho_{m'}, r')\geq 2 \log(2/\delta) \hat \Delta_{k^*_u} (\rho_{m'}, r')^{-2}, \nonumber
    \end{align}
    which implies that, the arm $k^*_u$ has been sampled more than $ 2 \log(2/\delta) \hat \Delta_{k^*_u} (\rho_{m'}, r')^{-2}$ before its eviction.  By assumption we therefore have that
    \begin{align}
        T_{k^*_u}(i'_v,r'') \geq 2 \log(2/\delta) \hat \Delta_{k^*_u} (\rho_{m'}, r')^{-2}. \nonumber
    \end{align}
    This means that, a change point will be detected before the end of the episode $I'_v$. This leads then to a contradiction with the fact that, there is no detected change points in $I'_v$.

This completes the proof.
\end{proof}

\subsection{Regret analysis}\label{sec-regret-final-th2}
The purpose of this subsection is to upper bound the regret.  We discuss separately $\mathcal{M}_1$ and $\mathcal{M}_2$. 

\subsubsection{Regret of the intervals in $\mathcal{M}_1$}

Recall that for $u \in \mathcal{M}_1$, we denote by $k^*_u \in [K]$ an optimal arm on $I_u$ and $\mathcal{S}(i'_v, i'_{v+1})\neq \emptyset$, which means there exists an arm $\hat k_v \in \mathcal{S}(i'_v, i'_{v+1})$, for any $v\in \mathcal{C}$. Based on this observation, we decompose the regret with respect to arm $\hat k_v$ and denote the regret on the set $\mathcal{M}_1$ by $R'(T)$.  To be specific, we have that
    \begin{align}
        R'(T) & = \mathbb E \left[\sum_{u \in \mathcal{M}_1} \sum_{r \in I_u} \sum_{k\in K_r} \left\{\Delta_k( \bar t_r(k)) + \Delta_{\hat k_v}( \bar t_r(\hat k_v)) - \Delta_{\hat k_v}( \bar t_r(\hat k_u))\right\}\right] \nonumber\\
        & \leq \mathbb E \left[\sum_{v \in \mathcal{C}} \sum_{k\in [K]} 2 \{\bar \Delta_{\hat k_v}(v)- \bar \Delta_k(v) + \bar \Delta_k(v)\}  T_k(i'_v, i'_{v+1}) \right]+ 2T(B^{*} \lor T^{-1/2}) \nonumber \\
        & = \mathbb E \left[\sum_{v \in \mathcal{C}} \sum_{k\in [K]} 2[ 2\{\bar \Delta_{k}(v)/2 - 2\bar \Delta_{\hat k_v}(v)\} + 4 \bar \Delta_{\hat k_v}(v) ]T_k(i'_v, i'_{v+1}) \right]+ 2T(B^{*} \lor T^{-1/2}) \nonumber\\
        & = \mathbb E \left[\sum_{v \in \mathcal{C}} \sum_{k\in [K]} \{4G^{(\hat k_v)}_k + 8\bar \Delta_{\hat k_v}(v)\}  T_k(i'_v, i'_{v+1}) \right]+ 2T(B^{*} \lor T^{-1/2}), \label{reg2:2}
    \end{align}
    where the first inequality is due to the following.  For any fixed $v \in \mathcal{C}$, there exists $u \in \mathcal{M}_1$ such that $I'_v \subset I_u$.  Then by Assumption~\ref{As:2-sec-obs-gap}, it holds on the interval $I'_v$ that $\Delta_k( \bar t_r(k)) \leq 2\bar \Delta_k(v)$ and $\Delta_{\hat k_v}( \bar t_r(\hat k_v)) \leq 2\bar \Delta_{\hat k_v}(v)$. Therefore we have that $-\Delta_{\hat k_v}( \bar t_r(\hat k_v)) \leq  -2 \bar \Delta_{\hat k_v}(v)$. 

Moreover, we further have
    \begin{align}
        \eqref{reg2:2} & =  \mathbb E \left[\sum_{v \in \mathcal{C}_1} \sum_{k\in [K]} \{4G^{(\hat k_v)}_k + 8\bar \Delta_{\hat k_v}(v)\}  T_k(i'_v, i'_{v+1}) \right] \nonumber \\
        & \hspace{1cm} + \mathbb E \left[\sum_{v \in \mathcal{C}_2} \sum_{k\in [K]} \{4G^{(\hat k_v)}_k + 8\bar \Delta_{\hat k_v}(v)\}  T_k(i'_v, i'_{v+1}) \right] + 2T(B^{*} \lor T^{-1/2}) \nonumber\\
        & = \mathbb E \left[\sum_{v \in \mathcal{C}_1} \sum_{k\in [K]} \{4G^{(\hat k_v)}_k + 8\bar \Delta_{\hat k_v}(v)\}  T_k(i'_v, i'_{v+1}) \right] \nonumber \\
        & \hspace{1cm} + \mathbb E \left[\sum_{v \in \mathcal{C}_2} \sum_{\substack{k\in [K] \\ \bar{\Delta}_k(v) = 0}} \{4G^{(\hat k_v)}_k + 8\bar \Delta_{\hat k_v}(v)\}  T_k(i'_v, i'_{v+1}) \right] \nonumber\\
        & \hspace{1cm} + \mathbb E \left[\sum_{v \in \mathcal{C}_2} \sum_{\substack{k\in [K] \\ \bar{\Delta}_k(v) > 0}} \{4G^{(\hat k_v)}_k + 8\bar \Delta_{\hat k_v}(v)\}  T_k(i'_v, i'_{v+1}) \right] + 2T(B^{*} \lor T^{-1/2}) \nonumber \\
        & = \mathbb E \left[\sum_{v \in \mathcal{C}_1} \sum_{k\in [K]} 4G^{(\hat k_v)}_k T_k(i'_v, i'_{v+1}) \right] + \mathbb E \left[\sum_{v \in \mathcal{C}_2} \sum_{\substack{k\in [K] \\ \bar{\Delta}_k(v) = 0}} 6\bar \Delta_{\hat k_v}(v)  T_k(i'_v, i'_{v+1}) \right] \nonumber\\
        & \hspace{1cm} + \mathbb E \left[\sum_{v \in \mathcal{C}_2} \sum_{\substack{k\in [K] \\ \bar{\Delta}_k(v) > 0}} 4G^{(\hat k_v)}_k   T_k(i'_v, i'_{v+1}) \right] + \mathbb E \left[\sum_{v \in \mathcal{C}_2} \sum_{\substack{k\in [K] \\ \bar{\Delta}_k(v) > 0}}  8\bar \Delta_{\hat k_v}(v) T_k(i'_v, i'_{v+1}) \right] \nonumber \\
        & \hspace{1cm} + 2T(B^{*} \lor T^{-1/2}) \nonumber \\ 
        & = \mathbb E \left[\sum_{v \in \mathcal{C}_1} \sum_{k\in [K]} 4G^{(\hat k_v)}_k T_k(i'_v, i'_{v+1}) \right] + \mathbb E \Bigg[\sum_{v \in \mathcal{C}_2}  \bar \Delta_{\hat k_v}(v) \Bigg\{ 6\sum_{\substack{k \in [K] \\ \bar{\Delta}_k(v) = 0}} T_k (i'_v, i'_{v+1}) \nonumber \\
        & \hspace{1cm} + 8 \sum_{\substack{k \in [K] \\ \bar{\Delta}_k(v) > 0}} T_k (i'_v, i'_{v+1})\Bigg\}\Bigg] + \mathbb E \left[\sum_{v \in \mathcal{C}_2} \sum_{\substack{k\in [K] \\ \bar{\Delta}_k(v) > 0}} 4G^{(\hat k_v)}_k   T_k(i'_v, i'_{v+1}) \right] \nonumber \\
        & \hspace{1cm}+ 2T(B^{*} \lor T^{-1/2}) \nonumber\\
        & = 4\underbrace{\mathbb E \left[\sum_{v \in \mathcal{C}} \sum_{k\in [K]}G^{(\hat k_v)}_k  T_k(i'_v, i'_{v+1})\right]}_{R_1(T)} + 14 \underbrace{\mathbb E \left[\sum_{v \in \mathcal{C}_2}  \bar \Delta_{\hat k_v}(v) |i'_{v+1}-i'_v|\right]}_{R_2(T)} + 2T(B^{*} \lor T^{-1/2}), 
    \end{align}
    where the third identity is due to the following.  As for the first term, note that for any $v\in \mathcal{C}_1$, arm $\hat k_v$ is an optimal arm.  As for the second term, note that in the interval $I_v$, $v\in \mathcal{C}_2$, it holds that $G^{\hat k_v}_{k^{*}_v}= -2 \bar \Delta_{\hat k_v}(v)$, since $k^{*}_u$ is optimal and $\bar \Delta_{k^{*}_u}(v)=0$.  The last identity follows from the fact that, on $\mathcal{C}_2$ the optimal arms $k^{*}_u$ will not be sampled on a regular basis in $I'_v$, which means that, $T_{ k^{*}_u}(i'_v, i'_{v+1}) \leq |i'_{v+1}-i'_v|$.  Furthermore, we also have that $T_k(i'_v, i'_{v+1}) \leq |i'_{v+1}-i'_v|$, since $k$ is an sub-optimal arm.  

\medskip    
\noindent \textbf{Case 1: Regret $R_1(T)$ incurred on $\mathcal{C}$ by the pulls of sub-optimal arms with respect to $\hat k_v$.} 

Following similar arguments as those in Section~\ref{Proof:th1:regret}, we first split the set of sub-optimal arms into
    \[
        \mathbb G_{1,v} = \left\{k\in [K]: \, G^{(\hat k_v)}_k < \sqrt{\frac{K}{|I'_v|}}\right\} \quad \mbox{and} \quad \mathbb G_{2,v} = \left\{k\in [K]: \, G^{(\hat k_v)}_k \geq \sqrt{\frac{K}{|I'_v|}}\right\}.
    \]
    Based on this decomposition, we have that
    \begin{align}
        R_1(T) = \mathbb E \left[\sum_{v \in \mathcal{C}} \sum_{k\in \mathbb G_{1,v}} G^{(\hat k_v)}_k  T_k(i'_v, i'_{v+1}) \right] + \mathbb E \left[\sum_{v \in \mathcal{C}} \sum_{k\in \mathbb G_{2,v}} G^{(\hat k_v)}_k  T_k(i'_v, i'_{v+1}) \right]. \label{reg2:eq2}
    \end{align}

As for $\mathbb G_{1,v}$, we have that
    \begin{align}
        & \mathbb E \left[\sum_{v \in \mathcal{C}} \sum_{k\in \mathbb G_{1,v}} G^{(\hat k_v)}_k  T_k(i'_v, i'_{v+1}) \right] \leq \mathbb E \left[\sum_{v \in \mathcal{C}} \sum_{k \in \mathbb G_{1,v}} \sqrt{\frac{K}{|I'_v|}} T_k(i'_v, i'_{v+1})\right] \leq  \sqrt{K} \mathbb E \left[\sum_{v \in \mathcal{C}} \sqrt{|I'_v|} \right] \nonumber \\
        & \leq \sqrt{KT}\mathbb E \left[\sqrt{|\mathcal{C}}|\right] \leq \sqrt{2KTM}, \label{reg2:eq3}
    \end{align}
    where the last inequality is due to Corollary~\ref{cor2}.

As for $\mathbb G_{2,v}$, we consider two scenarios, when a sub-optimal arm $k\in \mathbb G_{2,v}$'s gap is over-estimated and under-estimated.  We further decompose $\mathbb{G}_{2, v} = \mathbb{G}_{2, 1, v} \sqcup \mathbb{G}_{2, 2, v}$, depending on the conditions in Propositions~\ref{prop:6} and \ref{prop:7}, respectively.

We remark that in both Propositions~\ref{prop:6} and \ref{prop:7}, we require that $G^{(\hat{k}_v)}_k(v) \geq 4B^*$.  For $k$ and $v$ such that $G^{(\hat{k}_v)}_k(v) < 4B^*$, their contributions to the total regret is upper bounded by $4B^*T$.  For notational simplicity, in the rest of \textbf{Case 1}, we assume that $G^{(\hat{k}_v)}_k(v) \geq 4B^*$ holds.

\medskip
\noindent \textbf{Case 1.1: Over-estimated.}

Due to Proposition~\ref{prop:6}, we have that 
    \begin{align}
        & \mathbb E \left[\sum_{v \in \mathcal{C}} \sum_{k\in\mathbb  G_{2,1,v} } G^{(\hat k_v)}_k T_k(i'_v, i'_{v+1}) \right] \nonumber \\
        \leq & \mathbb E \left[\sum_{v \in \mathcal{C}} \sum_{k\in \mathbb G_{2,v}} G^{(\hat k_v)}_k \left\{2c^{-2} \log(2/\delta) \left(G_{k}^{(\hat k_v)}(v)\right)^{-2} + \frac{|I'_v|\sqrt{ M}}{\sqrt{TK}} \frac{1}{ G_{k}^{(\hat k_v)}(v)} + 1 \right\} \right]\nonumber\\
        \leq & \mathbb E \left[\sum_{v \in \mathcal{C}} \sum_{k\in \mathbb G_{2,1,v}}\left\{2 c^{-2} \log(2/\delta) \left(G_{k}^{(\hat k_v)}(v)\right)^{-1} + \frac{|I'_v|\sqrt{ M}}{\sqrt{T K}} + G_{k}^{(\hat k_v)}(v)\right\}\right] \nonumber\\
        \leq & \mathbb E \left[\sum_{v \in \mathcal{C}} \sum_{k\in \mathbb G_{2,1,v}}\left\{2 c^{-2} \log(2/\delta) \sqrt{\frac{|I'_v|}{K}} + \frac{|I'_v|\sqrt{M}}{\sqrt{T K}} + 1 \right\} \right] \nonumber \\
        \leq & 2c^{-2} \log(2/\delta)\sqrt{K} \mathbb E\left[\sum_{v \in \mathcal{C}} \sqrt{|I'_v|} \right] +  \frac{K}{\sqrt{T K}} \mathbb E \left[\sqrt{ M}  \sum_{v \in \mathcal{C}}  |I'_v|\right] + K \mathbb E[|\mathcal{C}|] \nonumber\\
        \leq & c^{-2} 2\log(2/\delta)\sqrt{KT} \mathbb E [\sqrt{ |\mathcal{C}|}]+ \sqrt{T K} \mathbb E [\sqrt{ M}]+ K \mathbb E[|\mathcal{C}|] \nonumber\\
        \leq & \left\{2^{3/2}c^{-2} \log(2/\delta) + 1\right\} \sqrt{KTM} + 2KM, \label{reg2:eq5}
    \end{align}
    where the third inequality is due to the definition of $\mathbb{G}_{2,v}$ and the fact that $G_{k}^{(\hat k_v)}(v) \leq 1$, and the final inequality follows from Corollary~\ref{cor2}.

\medskip
\noindent \textbf{Case 1.2: Under-estimated.}

Due to Proposition~\ref{prop:7}, we have that
    \begin{align}
        & \mathbb E \left[\sum_{v \in \mathcal{C}} \sum_{k\in\mathbb  G_{2,2,v} } G^{(\hat k_v)}_k T_k(i'_v, i'_{v+1}) \right] \leq \mathbb E \left[\sum_{v \in \mathcal{C}} \sum_{k\in \mathbb G_{2,2,v}} G^{(\hat k_v)}_k 16 \log(2/\delta) \left(G_{k}^{(\hat k_v)}\right)^{-2} \right] \nonumber\\
        \leq & \mathbb E \left[\sum_{v \in \mathcal{C}} \sum_{k\in \mathbb G_{2,2,v}}  16 \log(2/\delta) \left(G_{k}^{(\hat k_v)}\right)^{-1} \right] \leq \mathbb E \left[\sum_{v \in \mathcal{C}} \sum_{k\in \mathbb G_{2,2,v}}  16\log(2/\delta) \sqrt{\frac{|I'_v|}{K}} \right]\nonumber \\
        \leq & 16 \log(2/\delta)\sqrt{KT} \mathbb E \left[ \sqrt{|\mathcal{C}|}\right] \leq 16 \log(2/\delta)\sqrt{2KTM}, \label{reg2:eq6}
    \end{align}
    where the third inequality is due to the definition of $\mathbb{G}_{2,v}$ and $G_{k}^{(\hat k_v)} \leq 1$, and the final inequality is due to Corollary~\ref{cor2}.

\medskip
Plugging \eqref{reg2:eq3}, \eqref{reg2:eq5} and~\eqref{reg2:eq6} into \eqref{reg2:eq2}, it follows that
    \[
        R_1(T) \leq \left\{(2^{3/2}c^{-2} + 16)\log(2/\delta) + 2^{3/2}\right\} \sqrt{KTM} + 2KM.
    \]

\medskip
\noindent \textbf{Case2: Regret $R_2(T)$ incurred on $\mathcal{C}_2$.}

As for the term $R_2(T)$, it holds that
    \begin{align}
        & \mathbb E \left[\sum_{u \in \mathcal{C}_2} \bar \Delta_{\hat k_v}(v) |i'_{v+1}-i'_v|\right] \leq \mathbb E \left[\sum_{v \in \mathcal{C}_2} \bar \Delta_{\hat k_v}(v) 25 \log(2/\delta) \bar \Delta_{\hat k_v}(v)^{-1} \sqrt{\frac{TK}{M}}\right] \nonumber \\
        = & 25 \log(2/\delta) \sqrt{KT} \mathbb E \left[\sum_{v \in \mathcal{C}_2} \frac{1}{\sqrt{M}} \right] \leq 25 \log(2/\delta) \sqrt{KT} \mathbb E \left[\frac{|\mathcal{C}_2|}{\sqrt{M}} \right] \leq 25 \log(2/\delta) \sqrt{2KTM}, \nonumber
    \end{align}
    where the first inequality is due to Proposition~\ref{prop:8} and the last inequality is due to Corollary~\ref{cor2}.

Summing up the regret of both cases and taking $\delta=(KT^3)^{-1}$ lead us to 
    \begin{align}
        R'(T)\leq \left\{4\times (2^{3/2}c^{-2} + 366)\log(2/\delta) + 2^{11/2}\right\} \sqrt{KTM} + 8KM + 2T(B^{*} \lor T^{-1/2}) + 4TB^*,
    \end{align}
    for some $c < 1$.

\subsubsection{Regret of the intervals in $\mathcal{M}_2$}

This is identical as Section~\ref{Proof:th1:regret1} and the regret here satisfies that
    \begin{align}
        R''(T) \leq 4KM.
\end{align}

\vspace{1cm}
To sum up, the overall regret $R(T)$ can be obtained by summing the regrets $R'(T)$ and $R''(T)$.

\section{proof of the corollaries of Section~\ref{sec:conse}}\label{proof:conse}

\subsection{Proof of Corollary~\ref{cor-4}, for \textbf{Case b)}}

First note that Assumption~\ref{As:2-unobs} is satisfied for $B^* = u^*K/T$, by definition of $u^*$. We also have that on each segment $[\bar\tau_m^{(k)}, \bar\tau_{m+1}^{(k)})$, this polynomial setting is a special case of \textbf{Case d)}, see Subsection~\ref{ss:cased}. Indeed, the gaps $\Delta_k$'s are also polynomials of degree at most $\gamma^*$ - as the difference of two such polynomial. And a polynomial of degree $\gamma^*$ has at most $(\gamma^* - 1)\lor 0$ inflexion points. So this is a special case of \textbf{Case d)}, see Subsection~\ref{ss:cased}. And so for some partition in at most $M^*(\gamma^*+1) K(\lfloor \log_2(T^{1/2})\rfloor +1)$ intervals, Assumption~\ref{As:2-sec-obs-gap} is satisfied - see the construction of \textbf{Case d)} in Appendix~\ref{proof:cased} for more details. Therefore, we can apply Theorem~\ref{th:2} to provide a bound in this case on the regret of \CPDMAB.

\subsection{Proof of Corollary~\ref{cor:c}, for \textbf{Case c)}}

First note that Assumption~\ref{As:2-unobs} is satisfied for any $B^* \geq (K/T)^\alpha$, because of the H\"older assumption.  Also, since $f_k$'s are $\alpha$-H\"older continuous, the function $\max_k f_k$ is also $\alpha$-H\"older, so that the gaps satisfy that for any $t,t' \in [\bar\tau_m^{(k)}, \bar\tau_{m+1}^{(k)})$, we have $|\Delta_k(t) - \Delta_k(t')| \leq \left|\frac{t-t'}{T}\right|^{\alpha}$. And so for any $k,m$, and any $B^* > 0$, we define recursively the following change points for $\kappa \geq 1$, with $\zeta_{m,k,1}(B^*) = 1$:
$$\zeta_{m,k,\kappa+1}(B^*) = \min\left\{z \in [\zeta_{m,k,\kappa}(B^*),T): |\Delta_k(\zeta_{m,k,\kappa}(B^*)) - \Delta_k(z)| \geq B^* \right\}.$$
Define $\bar \kappa_{k,m}(B^*)$ as the total number of such change points for $k,m$. Note that on $[\bar\tau_m, \bar\tau_{m+1})$ and for arm $k$, Assumption~\ref{As:2-sec-obs-gap} is satisfied for $B^*$ and the partition induced by the change points $(\zeta_{m,k,\kappa+1}(B^*))_{\kappa \leq \bar \kappa_{k,m}(B^*)}$. 

Also because of the H\"older assumption, we know that $\bar \kappa_{k,m}(B^*)$ is bounded by $\lfloor  |\frac{\bar \tau_{m+1} - \bar \tau_m}{T}| (B^{*})^{-1/\alpha}\rfloor +1$ - so that the total number of all change points $(\zeta_{m,k,\kappa}(B^*))_{m\leq M, k\in [K], \kappa \leq \bar \kappa_{k,m}(B^*)}$ is smaller than $K(B^{*})^{-1/\alpha} + M^*$. And Assumption~\ref{As:2-sec-obs-gap} is satisfied for $B^*$ on the partition induced by all these change points.

Therefore, we can apply Theorem~\ref{th:2} to provide a bound in this case on the regret of \CPDMAB, for any $B^*$ such that $B^* \geq (K/T)^\alpha$, and with a number of intervals in the partition bounded as $K(B^{*})^{-1/\alpha} + M^*$ - and we optimise over $B^*$ to provide the following corollary.

\subsection{Proof of Corollary~\ref{cor21}, for \textbf{Case d)}}\label{proof:cased}

In this setting, (1) implies Assumption~\ref{As:2-unobs}. Also, (2) implies Assumption~\ref{As:2-sec-obs-gap}, but on a more refined partition than the one given by the $(\bar \tau_m)_m$. We describe here its construction. For any $m$ and any $k$, we know that $\Delta_k$ is monotone, let us assume non-decreasing for simplicity. Write for any $\lfloor \log_2(T^{1/2})\rfloor +1 \geq \kappa\geq 1$ 
$$\zeta_{k,m,\kappa} = \min\{t \in [\bar \tau^{(k)}_m, \bar \tau^{(k)}_{m+1}): \Delta_k(t) \geq 2^{\kappa} (B^* \lor T^{-1/2})\}.$$
By definition we have that the change points $(\zeta_{k,m,\kappa})_{\kappa\in \{1,\ldots, \lfloor \log_2(T^{1/2})\rfloor +1\}}$ define a partition of $[\bar \tau^{(k)}_m, \bar \tau^{(k)}_{m+1})$ in at most $\lfloor \log_2(T^{1/2})\rfloor +1$ intervals, such that the condition of Assumption~\ref{As:2-sec-obs-gap} is satisfied for this given $k$, and on $[\bar \tau^{(k)}_m, \bar \tau^{(k)}_{m+1})$, for the partition defined by the change points  $(\zeta_{k,m,\kappa})_{m\leq \upsilon^*, \kappa \leq \log_2(T^{1/2})\rfloor +1}$. Doing this for all $m\leq \upsilon^*$ and all $k \in [K]$, we have that the change points $(\zeta_{k,m,\kappa})_{k \in [K], m \leq \upsilon^*, \kappa\leq \log_2(T^{1/2})\rfloor +1}$ form a partition of $[T]$ in at most $\upsilon^*K(\lfloor \log_2(T^{1/2})\rfloor +1)$ intervals, such that Assumption~\ref{As:2-sec-obs-gap} is satisfied.

Therefore, we can apply Theorem~\ref{th:2} to provide a bound in this case on the regret of \CPDMAB.

\end{document}